\documentclass[twoside]{article}

\usepackage[accepted]{aistats2026}
\usepackage[utf8]{inputenc} %
\usepackage[T1]{fontenc}    %
\usepackage{hyperref}       %
\usepackage{url}            %
\usepackage{booktabs}       %
\usepackage{amsfonts}       %
\usepackage{nicefrac}       %
\usepackage{microtype}      %
\usepackage{xcolor}         %
\usepackage{algpseudocode}
\usepackage{algorithm}
\usepackage{dsfont}
\usepackage{amssymb}
\usepackage{amsthm}
\usepackage{graphicx}
\usepackage{subcaption}
\usepackage{dblfloatfix}
\usepackage{enumitem}
\usepackage[export]{adjustbox}%

\usepackage[round]{natbib}

\usepackage{amsmath,amsfonts,bm}

\def\eqref#1{equation~\ref{#1}}

\def\1{\bm{1}}
\newcommand{\train}{\mathcal{D}}

\def\vzero{{\bm{0}}}

\def\vmu{{\bm{\mu}}}

\def\vb{{\bm{b}}}

\def\vf{{\bm{f}}}
\def\vg{{\bm{g}}}

\def\vm{{\bm{m}}}

\def\vu{{\bm{u}}}
\def\vv{{\bm{v}}}
\def\vw{{\bm{w}}}
\def\vx{{\bm{x}}}
\def\vy{{\bm{y}}}
\def\vz{{\bm{z}}}

\def\mD{{\bm{D}}}

\def\mI{{\bm{I}}}

\def\mK{{\bm{K}}}
\def\mL{{\bm{L}}}

\def\mS{{\bm{S}}}

\def\mX{{\bm{X}}}

\def\mSigma{{\bm{\Sigma}}}

\DeclareMathAlphabet{\mathsfit}{\encodingdefault}{\sfdefault}{m}{sl}
\SetMathAlphabet{\mathsfit}{bold}{\encodingdefault}{\sfdefault}{bx}{n}

\newcommand{\E}{\mathbb{E}}

\newcommand{\R}{\mathbb{R}}

\DeclareMathOperator*{\argmax}{arg\,max}

\newcommand{\normal}{\mathcal N}

\newcommand{\inputspace}{\mathcal X}

\newcommand{\rkhs}{{\mathcal{H}}}

\makeatletter
\newtheorem*{rep@theorem}{\rep@title}
\newcommand{\newreptheorem}[2]{%
  \newenvironment{rep#1}[1]{%
    \def\rep@title{#2 \ref{##1}}%
  \begin{rep@theorem}}%
    {
\end{rep@theorem}}}
\makeatother

\newtheorem{theorem}{Theorem}
\newreptheorem{theorem}{Theorem}
\newtheorem{lemma}{Lemma}

\newcommand{\methodname}{Adaptive Candidate Thompson Sampling}
\newcommand{\methodnameshort}{ACTS}

\begin{document}

\twocolumn[

  \aistatstitle{Adaptive Candidate Point Thompson Sampling for High-Dimensional Bayesian Optimization}

  \aistatsauthor{Donney Fan \And Geoff Pleiss}

  \aistatsaddress{University of British Columbia \\ Vector Institute \\ \texttt{donneyf@cs.ubc.ca} \And  University of British Columbia \\ Vector Institute \\ \texttt{geoff.pleiss@stat.ubc.ca}}
]

\begin{abstract}
  
In Bayesian optimization, Thompson sampling selects the evaluation point
by sampling from the posterior distribution over the objective function maximizer.
Because this sampling problem is intractable for Gaussian process (GP) surrogates,
the posterior distribution is typically restricted to fixed discretizations (i.e., candidate points)
that become exponentially sparse as dimensionality increases.
While previous works aim to increase candidate point density through scalable GP approximations,
our orthogonal approach increases density by \emph{adaptively} reducing the search space during sampling.
Specifically, we introduce \methodname{} (\methodnameshort{}),
which generates candidate points in subspaces guided by the gradient of a surrogate model sample.
\methodnameshort{} is a simple drop-in replacement for existing TS methods---%
including those that use trust regions or other local approximations---%
producing better samples of maxima and improved optimization across synthetic and real-world benchmarks.

\end{abstract}
\section{INTRODUCTION}

Bayesian optimization (BO) is a popular framework for sample-efficient optimization, with prominent applications in
machine learning \citep{snoek2012practical,swersky2014freeze}, scientific discovery \citep{maus2022local,chitturi2024targeted,slattery2024automated},
and robotics \citep{berkenkamp2023bayesian,deisenroth2011pilco,deneault2021toward},
where a probabilistic surrogate model of a black-box function is refined iteratively through adaptively chosen function evaluations.
Historically, BO has been confined to low-dimensional problems (typically $\leq 10$ dimensions), but recent advances have scaled this paradigm to problems with hundreds or thousands of dimensions \citep{eriksson2019scalable,papenmeier2022increasing,hvarfner2024vanilla}.

Many of these new methods---including TuRBO \citep{eriksson2019scalable}, BAxUS \citep{papenmeier2022increasing}, and their variants \citep[e.g.][]{eriksson2021scalable,daulton2022multi,maus2022local,rashidi2024cylindrical}---%
rely on Thompson sampling (TS) \citep{thompson1933likelihood}
to determine which data points to acquire.
Unique among other acquisition functions,
Thompson sampling uses randomness to balance exploration and exploitation by
maximizing a \emph{sample} of the posterior surrogate model
rather than maximizing some deterministic function of its marginal moments.
This randomized strategy affords strong theoretical convergence guarantees \citep{russo2014learning,kandasamy2018parallelised}.

\begin{figure*}[t!]
  \centering
  \includegraphics[width=\linewidth]{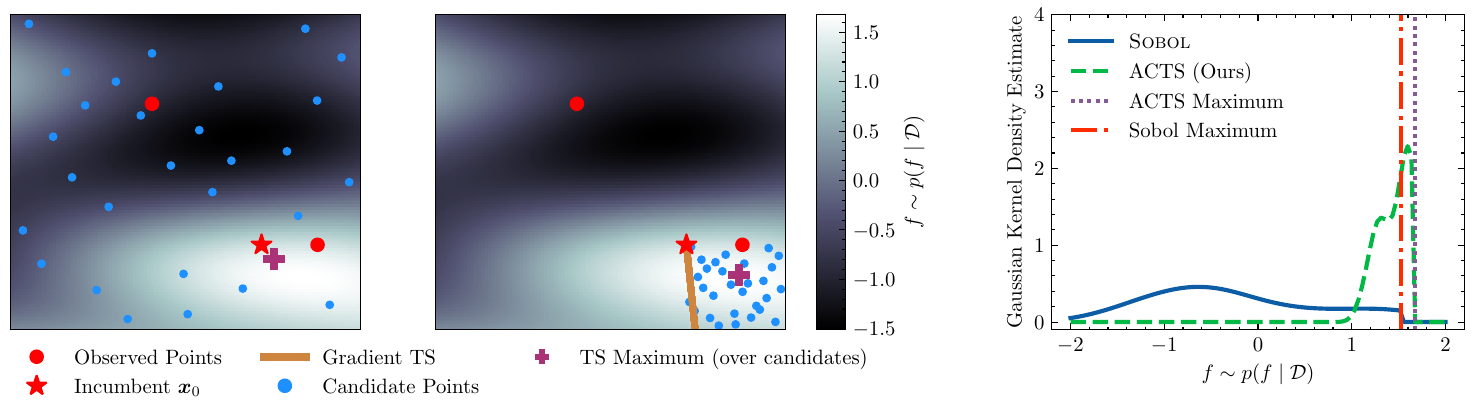}
  \vspace{-1em}
  \caption{
    {\bf Illustration of Adaptive Candidate Thompson Sampling.}
    A GP posterior sample (background colour) can only be evaluated at a finite set of candidate points.
    (Left.) Standard candidate point methods produce a poor discretization of the input domain.
    (Center.) ACTS only produces candidate points in a subspace that is aligned with the gradient of the posterior sample,
    increasing density in a region where a (local) function sample maximum is likely to be found.
    (Right.) Thus, ACTS candidate points yield values closer to the true optimum of the GP posterior sample.
  }
  \label{fig:gradient_region}
\end{figure*}

Applying TS with a Gaussian process (GP) surrogate is hampered by the curse of dimensionality.
As sampling a continuous path is intractable, the GP is often sampled on a discretized \emph{candidate set} of points, whose density is crucial for performance \citep{pleiss2020fast}.
However, the number of points needed to adequately fill a space grows exponentially with its dimension, a scaling that quickly overwhelms the computational budget of GP sampling (typically $\approx 10,000$ points \citep{garnett2023bayesian}).
Although scalable approximations can increase its efficacy \citep[e.g.][]{pleiss2018constant}, they do not overcome this exponential dependence.

While existing Thompson sampling methods use sparsity \citep{eriksson2019scalable,regis2013combining} or locality \citep{rashidi2024cylindrical} to mitigate the curse of dimensionality, this paper introduces a novel and complementary strategy.
Our approach, \methodname{} (\methodnameshort{}), \emph{adaptively} constructs candidate sets in regions where the GP sample path is likely to attain its maximum.
The key insight is that candidate sets need not be independent of the sample path.
\methodnameshort{} first samples the GP's gradient at the incumbent to identify an ascent direction, then populates a candidate set within a small, gradient-aligned region that is significantly smaller than the input space, enabling a far denser discretization of candidate points than na\"ive approaches.
This initial sample is extended to a joint sample over the new candidates, yielding a valid GP realization evaluated at promising points.
Despite using local gradient information to define the candidate set, we empirically show that \methodnameshort\ is no more local than other TS variants and theoretically show global consistency (i.e., it eventually queries arbitrarily close to the global maximizer).

We evaluate \methodnameshort{} across several high-dimensional benchmarks.
By coupling candidate point selection with posterior sampling,
\methodnameshort{} finds higher values of posterior sample paths than other TS methods.
When paired with local or sparse BO strategies,
\methodnameshort{} can offer significant improvements to optimization efficacy,
matching or exceeding alternative methods.

\section{Background}\label{section:background}
We seek to maximize an expensive black-box function $F(\vx): \R^d \to \R$ over a compact domain $\mathcal{X} \subset \mathbb{R}^d$ (often $[0,1]^d$). $F$ lacks a known
analytical form (e.g., no gradient information) and is accessed only through noisy point-wise evaluations of the form $y = F(\vx) + \sigma_n \epsilon$, where $\epsilon \sim \mathcal{N}(0, 1)$
and $\sigma_n$ is inferred from data.

\paragraph{Bayesian Optimization} maximizes $F$ by building a probabilistic surrogate model $f$, most commonly a Gaussian process (GP), which is updated as new observations are gathered.
The GP prior is $p(f) = \mathcal{GP}(0, k(\vx, \vx'))$, where we assume a zero mean function and $k: \inputspace \times \inputspace \to \R$ is a positive definite kernel function whose hyperparameters are typically estimated with type-II maximum likelihood \citep[][Ch.~5]{Rasmussen2006Gaussian}.
At each iteration $t$, given $n_t$ total observation pairs $\mathcal{D}_t = \{(\vx_i, y_i)\}_{i=1}^{n_t}$ (i.e., no gradient observations), the posterior over $f$ is also a GP, $p(f \mid \mathcal{D}_t) = \mathcal{GP}(\mu_t(\vx), k_t(\vx, \vx'))$, with mean and covariance
\begin{equation}
  \mu_t(\vx)=k(\vx, \mX_t) \widetilde\mK^{-1} \vy,
  \label{eqn:gp_posterior_mean}
\end{equation}
\begin{equation}
  k_t(\vx, \vx')=k(\vx, \vx')-k(\vx, \mX_t) \widetilde\mK^{-1} k(\mX_t, \vx').
  \label{eqn:gp_posterior_covar}
\end{equation}
Here, $\mX_t \in \mathbb{R}^{n_t \times d}$ is the row-wise concatenation of $\{\vx_i\}_{i=1}^{n_t}$, $k(\vx,\mX_t) = [k(\vx,\vx_i)]_{i=1}^{n_t}\in \mathbb{R}^{n_t}$,
$\widetilde \mK = [k(\vx_i,\vx_j)]_{ij} + \sigma^2 \mI \in \R^{n_t \times n_t}$ is the Gram matrix plus observational noise, and $\vy_t = [y_i]_{i=1}^{n_t}$.
Next, an \emph{acquisition function} $\alpha_{f \mid \mathcal{D}_t}$ is maximized to select the next point: $\vx_{t+1} = \argmax_{\vx \in \mathcal{X}}\alpha_{f \mid \train_t}(\vx).$
All acquisition functions balance exploration---searching over regions where the model is uncertain---with exploitation---searching near the best-known point, or \emph{incumbent} $\vx_0$:
\[
  {\textstyle \vx_0 := \argmax_{\vx \in \train_t} \E[f(\vx) \mid \train_t].}
\]
Common choices of $\alpha_{f \mid \train_t}$ include Expected Improvement (EI) \citep{movckus1975bayesian,jones1998efficient}, the Upper Confidence Bound (UCB) \citep{srinivas2010gaussian},
and Thompson Sampling \citep{thompson1933likelihood}; the latter of which is the focus of this work.

\paragraph{Thompson Sampling} is a randomized acquisition function, popular in high-dimensional BO \citep[e.g.][]{daulton2022multi,eriksson2019scalable,papenmeier2022increasing}, that maximizes a single function $f$ drawn from the posterior, $f \sim p(f \mid \mathcal{D}_t)$.
The next point is thus a sample of the posterior maximizer's location: $\vx_{t+1} \sim p(\argmax_{\vx \in \inputspace} f(\vx) \mid \train_t)$, corresponding to the policy $\alpha_{f \mid \train_t}(\vx) = f(\vx)$.
This approach provides a probabilistic balance between exploration and exploitation.
In regions of high posterior variance, function samples $f$ exhibit significant variability, driving the search towards unexplored areas.
At the same time, these samples will nearly interpolate the incumbent $\vx_0$, and so it is also likely that search is concentrated around a known good point.
TS is supported by theoretical guarantees of convergence to the global optimum under mild conditions \citep{kandasamy2018parallelised, russo2016information}.

\paragraph{Practical Implementations of TS with Candidate Points.}
Since sampling the infinite-dimensional function
$f \sim p(f \mid \mathcal{D}_t)$ is intractable, practical implementations sample over a smaller finite set of $M$ candidate points $\tilde{\mathcal{X}} = \{ \vx^*_1, \ldots, \vx^*_M \}$
via reparameterization:
\begin{equation}
  \setlength{\arraycolsep}{1pt}
  \vf_{\tilde{\mathcal{X}}} =:
  \begin{bmatrix} f(\vx^*_1) & \ldots & f(\vx^*_M)
  \end{bmatrix}
  = \mu_t(\tilde{\mathcal{X}}) + k_t(\tilde{\mathcal{X}}, \tilde{\mathcal{X}})^\frac{1}{2} \vz,
  \label{eqn:reparameterization}
\end{equation}
where $\vz \sim \mathcal{N}(\boldsymbol{0}, \mI)$, $\mu_t(\tilde{\mathcal{X}}) \in \R^M$ and $k_t(\tilde{\mathcal{X}}, \tilde{\mathcal{X}}) \in \R^{M \times M}$%
are the posterior mean and covariance (Eqs.~\ref{eqn:gp_posterior_mean}, \ref{eqn:gp_posterior_covar})
evaluated at the candidate points.\footnote{%
  $k_t(\tilde{\mathcal{X}}, \tilde{\mathcal{X}})^\frac{1}{2}$ denotes any $\mL$ such that $\mL \mL^\top = k_t(\tilde{\mathcal{X}}, \tilde{\mathcal{X}})$.
  The most common choice is the Cholesky factor.
}
Maximization is then restricted to this discretized domain:
\[
  {\textstyle
    \vx_{t+1} = \vx^*_{i_\mathrm{max}},
    \qquad i_\mathrm{max} := \argmax_{i=1,\ldots,M} [\vf_{\tilde{\mathcal{X}}}]_i.
  }
\]
Computing $k_t(\tilde{\mathcal{X}}, \tilde{\mathcal{X}})^\frac{1}{2}$ in Eq.~\ref{eqn:reparameterization} na\"ively scales with $O(M^3)$.
This cubic scaling limits $M$ to $10^4$ or less in practice,
even when using scalable sampling approximations \citep[e.g.][]{pleiss2018constant,pleiss2020fast, NIPS2007_013a006f, pmlr-v258-renganathan25a}.

\paragraph{Policies for Generating Candidate Points.}
Strategies for generating candidate points $\tilde{\mathcal{X}} \subset \inputspace$ can be formally described as random (or quasi-random) policies $\pi_0$ that place candidates within a subregion $C \subseteq \inputspace$ of the domain:
\begin{equation}
  \tilde{\mathcal{X}} \sim \pi_0(\tilde{\mathcal{X}}; C), \qquad C \subset \inputspace,
  \label{eqn:policy}
\end{equation}
Simple strategies include space-filling methods like \emph{Sobol sequences} or uniform sampling.
Policies such as \emph{Cylindrical Thompson Sampling (CTS)} \citep{rashidi2024cylindrical} sample from a spherical distribution around $\vx_0$, aiming to improve local exploration.
The most prominent policy for high-dimensional BO is \emph{Random Axis-Aligned Subspace Perturbations (RAASP)}
\citep{daulton2022multi,eriksson2019scalable,papenmeier2022increasing,papenmeier2025exploring}, which generates candidates by perturbing $\vx_0$ in only a few coordinates.
This policy can be implemented by sampling a perturbation vector $\vu_i$ which is multiplied by a sparse binary mask $\vb_i$ to limit the number of perturbed dimensions:
\begin{gather}
  \begin{aligned}
    \tilde{\vx}_i = \vx_0 + \vu_i \odot \vb_i, \\
    \left(\vx_0 + \vu_i\right) \sim \mathrm{Uniform}(\inputspace), \\
    [\vb_i]_j \overset{\mathrm{iid}}{\sim} \mathrm{Bernoulli}(20/d),
    \label{eqn:raasp}
  \end{aligned}
\end{gather}
where $\odot$ denotes element-wise multiplication (Hadamard product).
Each candidate $\tilde{\vx}_i$ is formed by applying a sparse binary mask $\vb_i$ to a perturbation vector $\vu_i$, which is typically derived from a Sobol sequence.
This effectively restricts perturbations of $\vx_0$ to a low-dimensional subspace.

While discretization makes TS feasible, it is hindered by the curse of dimensionality, as the number of points needed for adequate coverage grows exponentially with $d$.
For instance, Appendix~\ref{appendix:curse_dim} shows that even $10^6$ Sobol points can fail to sample near the optimum in simple settings.
The sparsity imposed by RAASP alleviates this issue to some degree,
but---as our experiments will show---RAASP can be improved upon through denser sampling afforded by our method.

\paragraph{Jacobian GP Model.} As GPs are closed under linear operations \citep[e.g.][]{Rasmussen2006Gaussian} they induce a joint Gaussian distribution over function and derivative values with any once-differentiable kernel (a condition that is satisfied by the popular RBF and Mat\'ern kernels). This allows us to obtain a joint prior over the observations $\vy_t = [y_1, \ldots, y_{n_t}]$, function values at the candidate points $f_{\tilde{\mathcal{X}}} = [f(\vx^*_1), \ldots, f(\vx^*_M)]$, and the gradient $\nabla f(\vx_0)$:
\begin{equation}
  \resizebox{0.91\columnwidth}{!}{$ %
    \begin{bmatrix} \vy_t \\ \vf_{\tilde{\mathcal{X}}} \\ \nabla f(\vx_0)
    \end{bmatrix}
    \sim
    \mathcal{N}\left(\mathbf{0},
      \begin{bmatrix}
        k(\mX,\mX) + \sigma_n^2\mI  & k(\mX, \tilde{\mathcal{X}}) & k(\mX,\vx_0) \nabla^\top \\
        k(\tilde{\mathcal{X}},\mX) & k(\tilde{\mathcal{X}},\tilde{\mathcal{X}}) & k(\tilde{\mathcal{X}},\vx_0) \nabla^\top \\
        \nabla  k(\vx_0, \mX) & \nabla k(\vx_0, \tilde{\mathcal{X}}) & \nabla k(\vx_0,\vx_0) \nabla^\top
    \end{bmatrix}\right),
  $}
  \label{eqn:gp_joint_gradient}
\end{equation}
where $\nabla$ and $\nabla^\top$ are the gradient operators with respect to the first and second arguments of $k$, respectively.
Applying standard Gaussian conditioning rules to Eq.~\ref{eqn:gp_joint_gradient} yields the joint posterior $p(\vf_{\tilde{\mathcal{X}}}, \nabla f(\vx_0) \,|\, \mathcal{D}_t)$.
This model has been used in BO with observed gradients \citep{wu2017bayesian, shekhar2021significance} or for local exploration \citep{muller2021local,nguyen2022local,wu2023behavior}.
To the best of our knowledge, our work is the first to leverage it within the context of TS and candidate point placement.

\paragraph{Other Thompson Sampling Approaches.}
Several alternative TS algorithms avoid constructing explicit candidate-point sets.
\citet{wilson2020efficiently,wilson2021pathwise}
propose a \emph{pathwise} sampling strategy,
where a GP is approximated by a finite-basis Bayesian linear regression model.
Sampling the model's finite-dimensional coefficients yields a differentiable function realization from the approximate posterior, which can be optimized using standard gradient-based methods.
However, this approach introduces an inexact GP approximation and its own curse of dimensionality, as the number of required basis functions can scale poorly with $d$.
Another strategy uses Markov-Chain Monte-Carlo (MCMC) to sample directly from the distribution of the maximizer, $p(\argmax_{\vx \in \inputspace} f(\vx) \mid \train_t)$ \citep{bijl2016sequential, sweet2024fast, yi2024improving}.
The efficiency of this direct approach is contingent on the MCMC algorithm, as many samplers have convergence rates that scale poorly in high-dimensional spaces.

\section{ADAPTIVE CANDIDATE THOMPSON SAMPLING}\label{section:method}

We introduce \methodname{}, a method that mitigates the curse of dimensionality in candidate-based Thompson sampling.
\methodnameshort{} samples the gradient of the posterior $f \sim p(f \mid \train_t)$ to concentrate the candidate set into a small region likely to contain a local maximum of $f$.
The full procedure is detailed in Algorithm~\ref{alg}.

\paragraph{Intuition.}
A GP posterior sample over a candidate set, $\vf_{\tilde{\mathcal{X}}} \sim p(\vf_{\tilde{\mathcal{X}}} \mid \train_t)$, can be drawn as the marginal of a joint sample that also includes an auxiliary variable, such as the gradient at the incumbent, $\nabla f(\vx_0)$:
\[
  \begin{bmatrix} \vf_{\tilde{\mathcal{X}}}^\top & \nabla f(\vx_0)
  \end{bmatrix}
  \sim p\left(
    \begin{bmatrix} \vf_{\tilde{\mathcal{X}}}^\top & \nabla f(\vx_0)
  \end{bmatrix} \mid \train_t \right).
\]
Since the joint posterior is Gaussian (Eq.~\ref{eqn:gp_joint_gradient}), we can factorize the distribution and sample sequentially using standard conditioning rules:
\begin{equation}
  \begin{aligned}
    \vf_{\tilde{\mathcal{X}}} \sim p(\vf_{\tilde{\mathcal{X}}} \mid \nabla f(\vx_0), \train_t),\\
    \nabla f(\vx_0) \sim p(\nabla f(\vx_0) \mid \train_t).
    \label{eqn:intermediate_sampling}
  \end{aligned}
\end{equation}
The key insight of our method is that the candidate set $\tilde{\mathcal{X}}$ need not be fixed beforehand, but can be \emph{adaptively} constructed based on the sampled gradient:
\begin{equation}
  \begin{aligned}
    \vf_{\tilde{\mathcal{X}}_t} \sim p(\vf_{\tilde{\mathcal{X}}_t} \mid \tilde{\mathcal{X}}_t, \train_t),
    \\
    \tilde{\mathcal{X}}_t \sim \pi(\tilde{\mathcal{X}} \mid \nabla f(\vx_0)),
    \\
    \nabla f(\vx_0) \sim p(\nabla f(\vx_0) \mid \train_t).
    \label{eqn:adaptive_sampling}
  \end{aligned}
\end{equation}
Here, $\pi$ is a policy that uses the sampled gradient $\nabla f(\vx_0)$, a direction of ascent for the function sample $f$.
Consequently, $\pi$ can densely place candidate points in this promising region, increasing the likelihood of a large Thompson sample $\max_{\vx \in \tilde{\mathcal{X}}} f(\vx)$.
We emphasize that the resulting $\vf_{\tilde{\mathcal{X}}_t}$ in Eq.~\ref{eqn:adaptive_sampling} is ``exact''---i.e., it is a valid realization of $p(f \mid \train_t)$ on some finite subset---%
even though $\tilde{\mathcal{X}}$ is adaptively chosen.

\begin{algorithm}[t!]
  \caption{\methodname{} (\methodnameshort{})}
  \hspace*{\algorithmicindent}
  Hyperparameters: Number of candidate points $M$.
  \begin{algorithmic}[1]
    \State $\vx_0 \gets \vx_j$, $j = \argmax_i y_i$ \Comment{Obtain incumbent}
    \State $\nabla f(\vx_0) \sim \nabla f(\vx_0) \mid \train_t$\Comment{Sample via Eq. \ref{eqn:gp_joint_gradient}}.
    \State $[\vx^*_1, \ldots, \vx^*_M] =: \tilde{\mathcal{X}}_t \sim \pi( \tilde{\mathcal{X}} \mid \nabla f(\vx_0) )$%
    \label{alg:variable_line} \Comment{Generate candidate points}
    \State $\vf_{\tilde{\mathcal{X}}_t} \sim p(\vf_{\tilde{\mathcal{X}}_t} \mid \nabla f(x_0), \train_t)$
    \Comment{Sample via Eq. \ref{eqn:gp_joint_gradient}}
    \State $\vx_{t+1} = \vx^*_{i_\mathrm{max}}$, where $i_\mathrm{max} = \argmax_i [\vf_{\tilde{\mathcal{X}}_t}]_i$
  \end{algorithmic}
  \label{alg}
\end{algorithm}

\paragraph{\methodnameshort\ Search Spaces.}
Our adaptive candidate policy, $\pi(\tilde{\mathcal{X}} \mid \nabla f(\vx_0))$, considers a general recipe that applies a base non-adaptive policy $\pi_0$ (see Eq.~\ref{eqn:policy}) to a smaller search space $\mathcal{T}_{\nabla f(\vx_0)} \subset \inputspace$ that is aligned with the sampled gradient:
\begin{equation}
  \pi(\tilde{\mathcal{X}} \mid \nabla f(\vx_0)) := \pi_0(\tilde{\mathcal{X}}; \mathcal{T}_{\nabla f(\vx_0)}).
  \label{eqn:adaptive_policy}
\end{equation}
Specifically, \methodnameshort{} defines $\mathcal{T}_{\nabla f(\vx_0)}$ as an axis-aligned cone rooted at the incumbent $\vx_0$ and extending along the coordinate-wise direction of the gradient:
\begin{equation}
  \mathcal{T}_{\nabla f(\vx_0)} = \left\{\vx_0 + \vv\odot\nabla f(\vx_0) \mid \mathbf{0} \preceq \vv \in \mathbb{R}^d \right\} \cap \inputspace,
  \label{eqn:gradient_region}
\end{equation}
where $\odot$ denotes element-wise product.
This construction, illustrated in Fig.~\ref{fig:gradient_region}, defines a $d$-dimensional search space (rectangular when $\inputspace$ is rectangular) that is significantly smaller than the original domain.
For instance, if $\vx_0$ is centered in a $d=100$ dimensional space, $\mathrm{vol}\left(\mathcal{T}_{\nabla f(\vx_0)}\right)$ is reduced by a factor of $2^{100} \approx 10^{30}$, as we remove a half-space for each dimension.
Importantly, by aligning with the gradient, $\tilde{\mathcal{X}}$ is guaranteed to contain
points with higher function values than the incumbent $\vx_0$,
as any $\mathcal{T}_{\nabla f(\vx_0)}$ with non-zero hypervolume contains some $\vx'$ with $f(\vx') > f(\vx_0)$.
While the global $\argmax_{\vx \in \inputspace} f(\vx)$ may not lie in $\mathcal{T}_{\nabla f(\vx_0)}$,
it is likely that $\mathcal{T}_{\nabla f(\vx_0)}$ contains some local maximum of $f$.
Thus, limiting the search to $\mathcal{T}_{\nabla f(\vx_0)}$ increases the likelihood that some $\vx^* \in \tilde{\mathcal{X}}$ is close to this maximizer.

Although this strategy confines the candidate set to a region based on local gradient information, it avoids getting stuck in local optima because this gradient is a random draw.
We prove in Appendix \ref{appendix:consistency} that \methodnameshort{} maintains the global consistency of standard TS.

\begin{theorem}[Informal]
  Under mild assumptions, running Bayesian optimization with \methodnameshort\ is guaranteed to eventually query a point arbitrarily close to a global maximizer.
  \label{thm:global_consistency}
\end{theorem}

The proof of Theorem \ref{thm:global_consistency} shows that since we do not observe gradients, the posterior covariance never fully collapses, and thus a sampled gradient is non-zero almost surely, thus the ACTS search space remains well-defined, even when $\nabla f(\vx_0) \approx \mathbf{0}$.

\emph{Alternative search spaces.}
The axis-aligned cone is one of several possible geometries for $\mathcal{T}_{\nabla f(\vx_0)}$. We explore alternatives in Appendix~\ref{appx:ablation_search_space}, including a one-dimensional line search along the gradient.
We find that the construction in Eq.~\ref{eqn:gradient_region} provides a robust balance between reducing the search space volume and yielding effective acquisitions.

\paragraph{\methodnameshort{} with RAASP.}
Having defined the search space $\mathcal{T}_{\nabla f(\vx_0)}$, \methodnameshort{} can enhance any base policy $\pi_0$, such as Sobol sampling \citep[e.g.][]{balandat2020botorch} or CTS \citep{rashidi2024cylindrical}, by applying it within this restricted region (Eq.~\ref{eqn:adaptive_policy}). Here, we focus on integrating \methodnameshort{} with RAASP, which is arguably among the most widely used policies in high-dimensional BO \citep{eriksson2019scalable,papenmeier2022increasing}.

While the standard RAASP policy (Eq.~\ref{eqn:raasp}) can be directly plugged into Eq.~\ref{eqn:adaptive_policy},
we provide a simple modification that additionally leverages information from $\nabla f(\vx_0)$.
Recall that RAASP produces candidates by randomly perturbing the incumbent $\vx_0$ in $\approx 20$ randomly chosen dimensions.
With the gradient information from ACTS, we can favour perturbations in dimensions aligned with $\nabla f(\vx_0)$
(i.e., the direction of steepest ascent).
Mathematically, we modify the masking vector $\vb_i \in \{0, 1\}^d$ in Eq.~\ref{eqn:raasp}---%
which selects the dimensions to perturb---%
from i.i.d. Bernoulli entries to non-i.i.d. entries:
\begin{equation}
  [\vb_i]_j \overset{\mathrm{iid}}{\sim} \mathrm{Bernoulli}\left(\min\left\{20\tfrac{[\nabla f(\vx_0)]_j^2}{||\nabla f(\vx_0)||_2^2}, 1\right\}\right).
  \label{eqn:adaptive_raasp}
\end{equation}
While RAASP may produce perturbations in dimensions with minimal gradient (i.e., flatter directions with low likelihood of a posterior sample maximizer),
our formulation ensures that dimensions with greater contributions to the gradient vector are more likely to be perturbed.

\paragraph{Compatibility with Local BO Methods.}
\methodnameshort{} is readily compatible with local BO methods that employ trust regions, such as TuRBO \citep{eriksson2019scalable}.
Letting $\mathcal{R}_t$ denote the trust region at time $t$,
ACTS simply replaces the non-adaptive candidate policy $\pi_0(\tilde{\mathcal{X}}; \mathcal{R}_t)$
with $\pi_0(\tilde{\mathcal{X}}; \mathcal{R}_t \cap \mathcal{T}_{\nabla f(\vx_0)})$.
The resulting search space is smaller than the original trust region $\mathcal{R}_t$.
While both methods restrict the search space, they are different and complementary.
Trust region strategies enhance regression accuracy within a confined area to rapidly converge to a local optimum.
In contrast, \methodnameshort{} improves the fidelity of the Thompson sampling acquisition by more densely discretizing the posterior sample $f$ in a region likely to contain its maximum.
Again, the global consistency guarantee (Theorem~\ref{thm:global_consistency}) ensures that this focus on a subregion of $\mathcal X$ does not yield convergence to a local optimum.
This makes \methodnameshort{} a purposeful search policy that is orthogonal to (local) trust region approaches.

\paragraph{Compatibility with Batch BO.}
Thompson sampling naturally extends to batch BO, where $q$ points are acquired by maximizing $q$ independent posterior samples, $f^{(1)}, \dots, f^{(q)} \sim p(f \mid \train_t)$.
\methodnameshort{} integrates seamlessly into this framework.
For each sample $f^{(i)}$, we independently perform the \methodnameshort{} procedure: draw a gradient $\nabla f^{(i)}(\vx_0)$, construct a candidate set $\tilde{\mathcal{X}}_t^{(i)}$ within the corresponding cone $\mathcal{T}_{\nabla f^{(i)}(\vx_0)}$, and find the maximizer of $f^{(i)}$ over this unique set.
This process naturally promotes diversity.
Since the initial gradient samples are random, the adaptive cones and resulting candidate sets typically differ, directing the parallel searches to distinct promising regions.

\paragraph{Computational Complexity.}
While \methodnameshort{} requires additional computation over standard TS methods,
this overhead is negligible under a fixed candidate point budget.
All methods take $O(M^3)$ time to sample the GP posterior over $\tilde{\mathcal{X}}$.
Given $| \train_t | = n_t$ observations,
ACTS uses an additional $O(n_t^2 d)$ computation to sample from $p(\nabla f(\vx_0) \mid \train_t)$
(i.e., via standard conditioning of Eq.~\ref{eqn:gp_joint_gradient})
and $O(M d^2)$ to condition $p(\vf_{\tilde{\mathcal{X}}_t} \mid \train_t)$ on $\nabla f(\vx_0)$ to sample $\vf_{\tilde{\mathcal{X}}_t}$.
Assuming $d \ll M$ and $n_t \ll M$, as is often the case,
ACTS retains the $O(M^3)$ asymptotic complexity of standard TS methods.
It also permits scalable sampling extensions as mentioned in Section \ref{section:background}.

\section{EXPERIMENTAL RESULTS}\label{section:experiments}

\paragraph{Baselines.}
We compare \methodnameshort{} against three leading TS baselines: \textsc{RAASP} \citep{eriksson2019scalable} and \textsc{Cylindrical TS} \citep{rashidi2024cylindrical} candidate policies, and the candidate-free \textsc{Pathwise TS} \citep{wilson2020efficiently}.
Each method is evaluated both in a global space and with TuRBO trust regions.
We omit the naive \textsc{Sobol} policy from main comparisons due to its poor performance (Appendix \ref{appendix:curse_dim}), though we retain it for ablation studies.
On large-scale benchmarks, we add comparisons to three state-of-the-art methods: the sparse, fully-Bayesian SAASBO \citep{eriksson2021high}; the progressive latent-space method BAxUS \citep{papenmeier2022increasing}; and the improved \textsc{LogEI} \citep{ament2023unexpected}.

\paragraph{Implementation Details.}
We use GP surrogates with a constant mean and a squared exponential kernel, adopting the dimensional-scaled priors from \citet{hvarfner2024vanilla}.
All candidate-based TS methods use $M=10^4$ points.
For the GuacaMol benchmarks \citep{brown2019guacamol}, which involve thousands of evaluations, we use conjugate gradients \citep{gardner2018gpytorch} to efficiently compute the marginal log-likelihood.
Each run is warm-started with 30 initial points from a Sobol sequence.
Results are averaged over 10 runs with plots showing mean performance shaded by $\pm$2 standard errors.
Further details are in Appendix \ref{appendix:experiments} and an implementation of \methodnameshort{} is available at \url{https://github.com/DonneyF/ACTS}.

\begin{figure}[t]
  \centering
  \includegraphics[width=\linewidth]{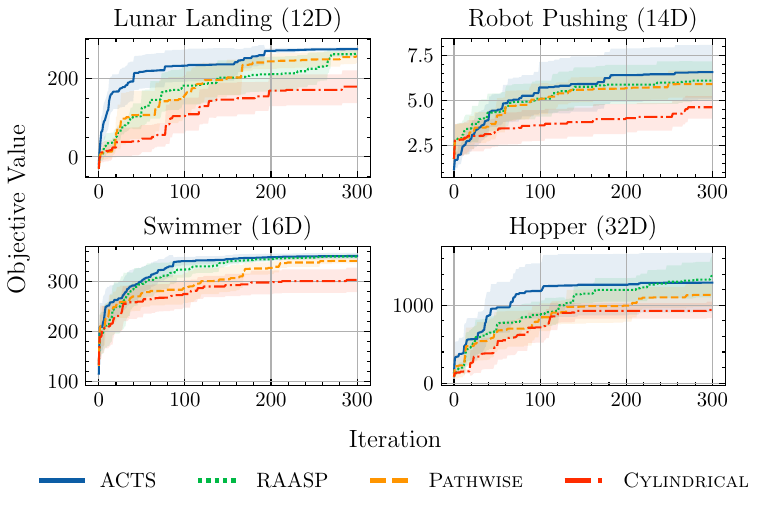}
  \vspace{-1em}
  \caption{{\bf Optimization performance of medium-dimensional real-world optimization problems.}   \methodnameshort\ consistently exhibits top performance (within two standard errors) and achieves high objective values earlier in optimization than other methods.}
  \label{fig:medium_dimensional}
\end{figure}

\subsection{Optimization Performance}
\begin{figure*}[t]
  \centering
  \includegraphics[width=\linewidth]{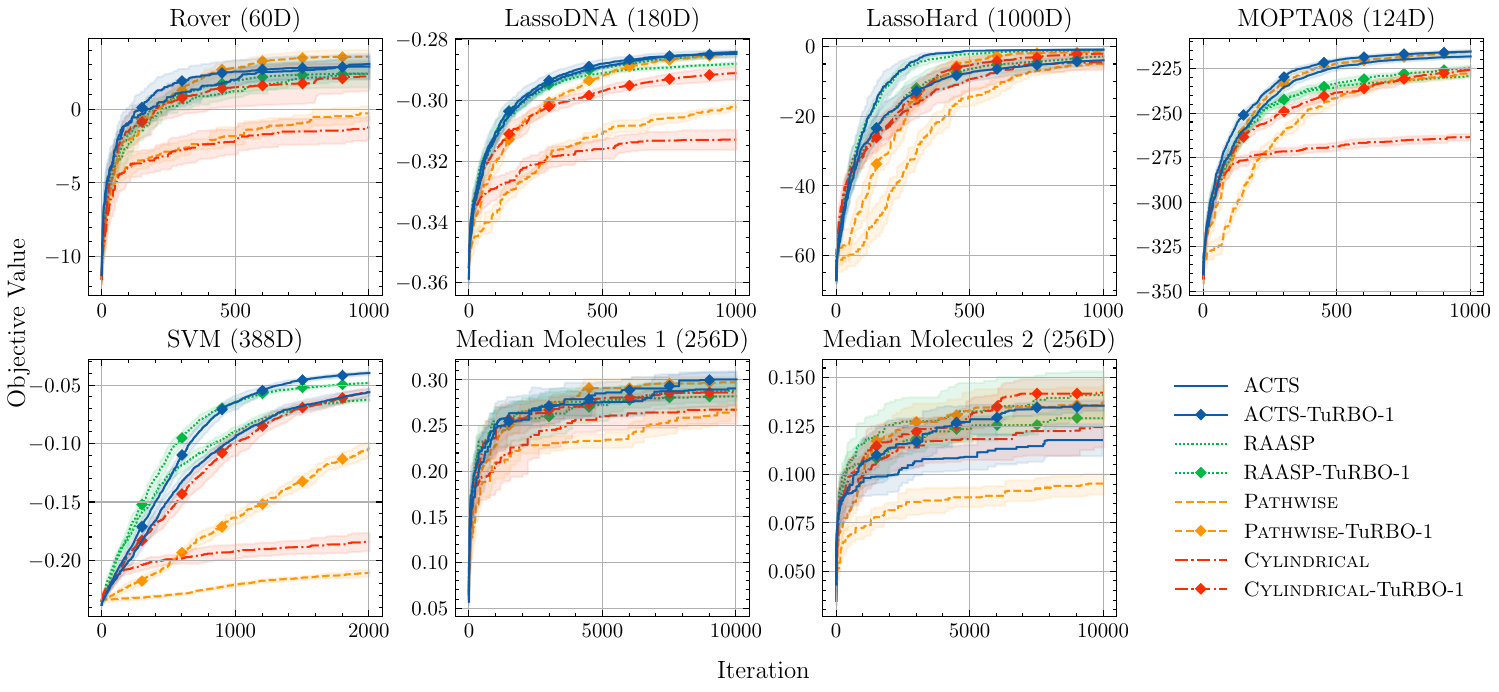}
  \vspace{-1em}
  \caption{
    {\bf Optimization performance on several high-dimensional problems.}
    \methodnameshort\ ranks highly in all benchmarks, achieving top performance (within two standard errors) on all benchmarks,
    and outperforming RAASP (with significance) on MOPTA08, SVM, and Median Molecules 1.
    (Note that TuRBO does not consistently improve performance on all benchmarks, but the performance boost/regression it induces is consistent across all methods on a given benchmark.)
  }
  \label{fig:high_dimensional}
\end{figure*}
\begin{table*}[t]
  \centering
  \resizebox{\linewidth}{!}{
    \begin{tabular}{cccccccc}
  \toprule
  &   & ACTS-TuRBO (Ours) & RAASP-TuRBO & \textsc{Pathwise} & \textsc{LogEI} & SAASBO & BAxUS \\
  \midrule
  Rover (60D) &   & $\mathbf{ 2.87 \pm 0.85 }$ & $\mathbf{ 2.36 \pm 0.94 }$ & $-0.33 \pm 0.49$ & $\mathbf{ 3.09 \pm 0.44 }$ & $\mathbf{ 2.42 \pm 1.01 }$ & $0.67 \pm 1.29$ \\
  LassoDNA (180D) &   & $\mathbf{ -0.28 \pm 0.00 }$ & $\mathbf{ -0.28 \pm 0.00 }$ & $-0.30 \pm 0.00$ & $-0.29 \pm 0.00$ & $-0.29 \pm 0.00$ & $-0.29 \pm 0.00$ \\
  MOPTA08 (124D) &   & $\mathbf{ -215.81 \pm 1.07 }$ & $-225.12 \pm 1.46$ & $-227.79 \pm 1.51$ & $-217.94 \pm 0.67$ & $\mathbf{ -219.16 \pm 2.34 }$ & $-240.85 \pm 2.30$ \\
  SVM (388D) &   & $\mathbf{ -0.04 \pm 0.00 }$ & $-0.05 \pm 0.00$ & $-0.21 \pm 0.00$ & $-0.05 \pm 0.00$ & $-0.15 \pm 0.02$ & $-0.10 \pm 0.00$ \\
  M. Mol. 1 (256D) &   & $\mathbf{ 0.30 \pm 0.01 }$ & $\mathbf{ 0.28 \pm 0.02 }$ & $0.27 \pm 0.01$ & $\mathbf{ 0.30 \pm 0.01 }$ & --- & $\mathbf{ 0.29 \pm 0.01 }$ \\
  M. Mol. 2 (256D) &   & $0.14 \pm 0.00$ & $0.13 \pm 0.00$ & $0.10 \pm 0.01$ & $\mathbf{ 0.19 \pm 0.01 }$ & --- & $\mathbf{ 0.20 \pm 0.01 }$ \\
  \bottomrule
\end{tabular}

  }
  \caption{Final objective  values from \methodnameshort\ vs. state-of-the-art methods on high-dim. benchmarks.}
  \label{tab:high_dimensional}
\end{table*}

\paragraph{Medium-Dimensional Problems.}
We first evaluate \methodnameshort{} on medium-dimensional ($d \leq 32$) benchmarks from robotics and control: Lunar Lander (12D), Robot Pushing (14D) \citep{wang2018batched}, Swimmer (16D), and Hopper (32D) \citep{todorov2012mujoco}.
As shown in Fig.~\ref{fig:medium_dimensional}, after 300 iterations of optimization, \methodnameshort{} achieves the highest reward on most tasks, frequently outperforming both RAASP and CTS.

\paragraph{High-Dimensional Optimization Performance.}
We evaluate our method on several high-dimensional problems ($d \leq 1000$) common in the literature.
These include Rover (60D) \citep{wang2018batched}, MOPTA08 (124D) \citep{jones2008large}, an SVM hyperparameter tuning task (388D) \citep{papenmeier2022increasing}, and tasks from the LassoBench suite (180--1000D) \citep{vsehic2022lassobench}.
Finally, we consider challenging molecule design tasks from the GuacaMol benchmarks \citep{brown2019guacamol}, optimizing within a 256D continuous latent space defined by a pretrained SELFIES-VAE \citep{maus2022local}.

Figure~\ref{fig:high_dimensional} compares the four TS approaches, both in a global setting and within TuRBO trust regions \citep{eriksson2019scalable}.
While no single baseline is consistently best, \methodnameshort{} is a top performer across nearly all benchmarks, with or without TuRBO.
(In Appendix~\ref{appendix:seq_opt}, we perform a ranking analysis to confirm that ACTS outperforms all other methods in aggregate across benchmarks.)
\textsc{Pathwise} and \textsc{Cylindrical TS} have varying levels of efficacy and TuRBO is sometimes less efficient than using the entire space.
While RAASP offers competitive performance on some tasks, \methodnameshort{} demonstrates a clear advantage on MOPTA08, SVM, and Median Molecules 1.

We further compare \methodnameshort{} against SAASBO, BAxUS, and \textsc{LogEI}.\footnote{We omit SAASBO on the GuacaMol benchmarks due to its prohibitive runtime.}
For brevity, Table~\ref{tab:high_dimensional} reports the final objective value achieved by each method (see Appendix~\ref{appendix:seq_opt} for full optimization plots).
\methodnameshort{} consistently matches or exceeds the performance of these state-of-the-art methods, demonstrating robust performance where other baselines are less consistent.

\begin{figure}[t]
  \centering
  \includegraphics[width=\linewidth]{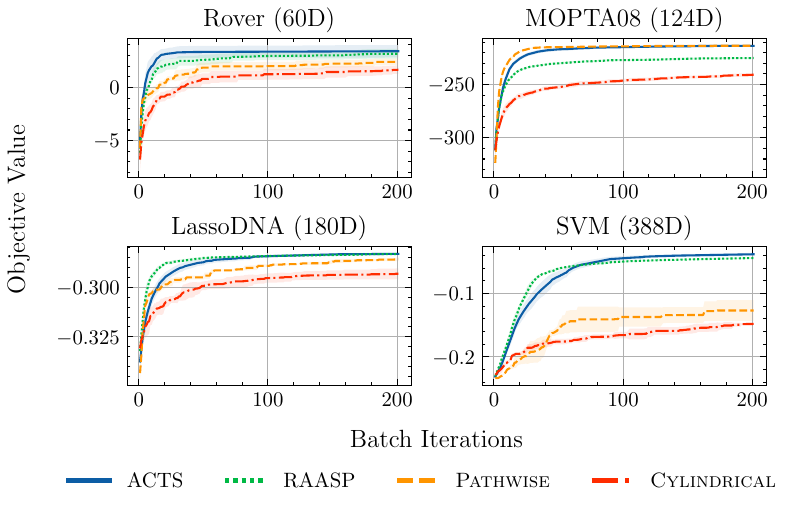}
  \vspace{-1em}
  \caption{{\bf Batch optimization performance ($q=100$) of selected high-dimensional problems.} \methodnameshort\
  achieves top results (within two standard errors) on all benchmarks.}
  \label{fig:batch}
\end{figure}

\paragraph{Batch Optimization.}
We evaluate \methodnameshort{} in a parallel optimization setting, a common application for Thompson sampling.
Figure~\ref{fig:batch} compares the four TS methods (without TuRBO) on a subset of benchmarks using a batch size of $q=100$ for 200 iterations.
We observe that \methodnameshort{} maintains its strong performance ranking relative to the other TS methods, demonstrating its effectiveness extends from the sequential to the batch setting.
Additional results for smaller batch sizes ($q=10, 50$) are in Appendix~\ref{appendix:batch_opt}.

\subsection{Analysis and Ablations}

\begin{figure}[t]
  \includegraphics[width=\linewidth, valign=c]{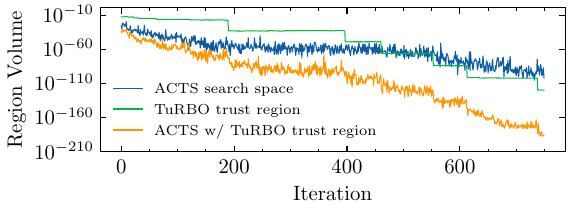}
  \caption{
    {\bf \methodnameshort\ increases candidate point density by shrinking search spaces.} \methodnameshort\ exhibits volumes similar to TuRBO without the direct use of trust regions in two BO runs of Rover, with and without TuRBO, up to the first restart.
  }
  \label{fig:subspace_volume}
\end{figure}

\paragraph{Analysis of Search Space Volume.}
We hypothesize that \methodnameshort{} improves performance by increasing candidate point density within a drastically smaller search space.
We verify this in Fig.~\ref{fig:subspace_volume}, which plots the search space volume induced by \methodnameshort{} on the 60D Rover problem over $\mathcal{X} = [0,1]^d$.
The volume is orders of magnitude smaller than the full domain, confirming a massive increase in candidate density under a fixed budget.
In this case, we have that $0.5^{60} \approx 10^{-18}$; thus ACTS reduces on average each dimension by more than half.
When combined with TuRBO, the intersection of the two regions shrinks the volume even further (to as low as $10^{-200}$), confirming that the methods provide complementary benefits.

\begin{figure}[t]
  \centering
  \includegraphics[width=\linewidth]{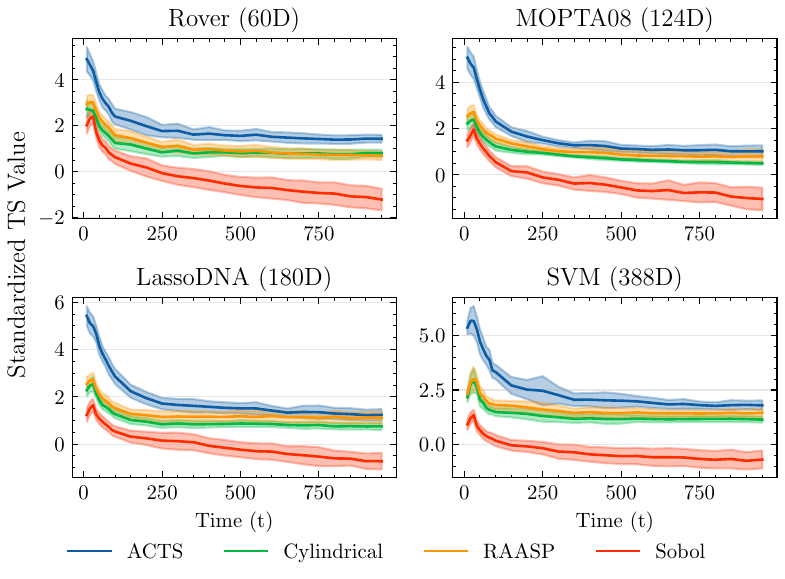}
  \caption{
    \methodnameshort\ produces greater TS maxima than other candidate policies on several benchmarks. We generate candidates from each policy, sample from the posterior 100 times, and aggregate over 10 models with standardized outputs.
  }
  \label{fig:sample_quality}
\end{figure}
\paragraph{Analysis of Candidate Points Function Values.}
We hypothesize that the \methodnameshort\ discretization will better cover regions of high function values,
and will thus produce candidates closer to the true posterior sample maximum.
To this end, we compare the $f_\text{max}$ distributions produced by \methodnameshort\ versus other candidate point methods on the 60-dimensional Rover task.
Specifically, we seed each method with the same $\vert \mathcal{D}_t \vert = 200$ observations and GP hyperparameters,%
\footnote{We generate this dataset and hyperparameters from 200 iterations of RAASP-based TS.}
generate $10^4$ candidate points, and draw $10^4$ Thompson samples from the GP posterior.
In Fig.~\ref{fig:sample_quality} we plot the histogram of $\max_{\vx \in \tilde{\mathcal{X}}} f(\vx)$ values produced by each method over 500 seeds of this procedure.
As we hypothesized, \methodnameshort\ produces larger values of $\max_{\vx \in \tilde{\mathcal{X}}} f(\vx)$ than other methods.
Moreover, these higher sample function values are correlated with higher objective function values.
The bottom row displays the true objective function values for $\argmax_{\vx \in \tilde{\mathcal{X}}} f(\vx)$ points of each method,
and \methodnameshort\ achieves the highest objective values.

\begin{figure}[!ht]
  \centering
  \includegraphics[width=\linewidth]{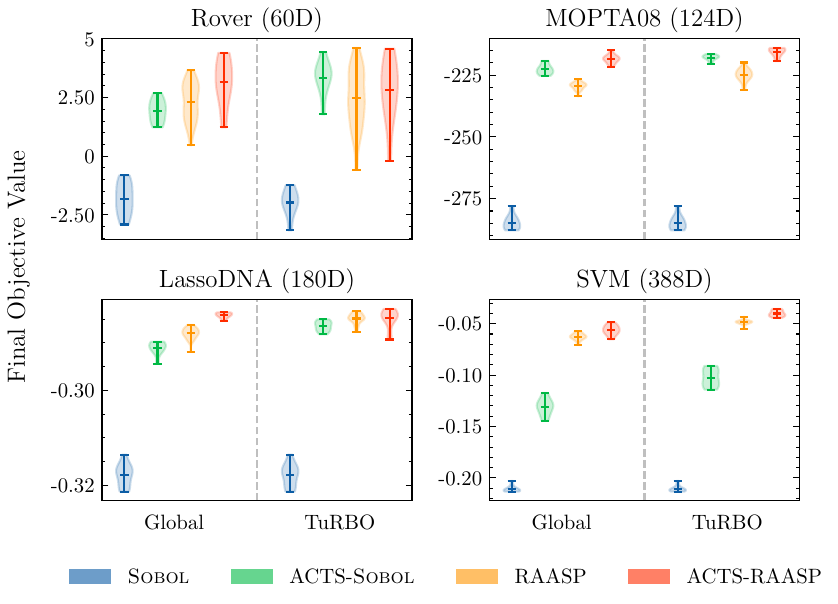}
  \vspace{-1em}
  \caption{
    {\bf Optimization performance of Sobol and RAASP policies with and without \methodnameshort\ on high-dimensional objectives.}
    RAASP consistently outperforms the poorly-performing Sobol policy.
    \methodnameshort\ search spaces provide a significant performance boost to all methods, especially Sobol. This indicates the \methodnameshort\ search spaces provide better fidelity for TS maximization than TuRBO trust regions.
  }
  \label{fig:gradient_uplift}
\end{figure}

\paragraph{Ablation Study of Base Policy.}
To test our hypothesis that \methodnameshort{} improves any underlying candidate strategy, we apply it to a Sobol base policy on the same benchmark tasks.
As shown in Figure~\ref{fig:gradient_uplift} (which includes RAASP variants for reference), \methodnameshort{} provides a significant performance uplift over standard Sobol candidates.
This improvement is not replicated by restricting the search space with TuRBO, demonstrating that the benefit stems from \methodnameshort{}'s specific candidate allocation
rather than any search space reduction.
While \methodnameshort{}-RAASP remains the top performer, \methodnameshort{} substantially closes the performance gap between the sophisticated RAASP and na\"ive Sobol policies.

\paragraph{Ablation Study of Search Space Geometry.}
In Appendix~\ref{appx:ablation_search_space}, we ablate our choice of search space by testing a \methodnameshort{} variant that restricts candidates to a one-dimensional line search along the sampled gradient.
This 1D alternative performs slightly worse than our proposed axis-aligned cone, suggesting the cone construction strikes a better balance between aggressively reducing the search space and avoiding an overly myopic search direction.

\paragraph{Search Behaviour.}

\methodnameshort{} substantially reduces the search space, which enables a high density of candidate points in promising regions.
While this might suggest a tendency toward local search, our analysis shows this is not the case; queries from \methodnameshort{} are often less local than TuRBO (with RAASP) or LogEI.
In Appendix~\ref{appendix:locality}, we compute the Traveling Salesman Problem tour for the sequence of queries,
a common metric for search locality \citep{papenmeier2025exploring},
and find that \methodnameshort\ often produces more exploratory trajectories than other methods.
We hypothesize the performance uplift of \methodnameshort\ stems from a higher-fidelity discretization of the Thompson sample.

\section{DISCUSSION}

\methodnameshort{} provides a principled approach to mitigate the curse of dimensionality in candidate-based TS.
By adaptively concentrating candidate points in a region aligned with a posterior gradient sample, our method achieves a denser, more effective discretization of promising areas.
Our experiments confirm that this strategy yields robust performance across a wide range of high-dimensional benchmarks in both sequential and batch settings, without biasing the search towards being any more local than existing TS methods.

\paragraph{Limitations and Extensions.}
As \methodnameshort\ is a drop-in replacement for standard TS approaches, it does not have many limitations beyond existing TS methods.
The primary assumption of ACTS is a once-differentiable kernel, and thus is suited towards objectives that are also once-differentiable.
We also note that reducing the search space based on a gradient is a myopic approximation of the maximizing direction,
which, despite global consistency, may be inappropriate in problem settings with low intrinsic lengthscales.
However, these situations may nevertheless benefit from other adaptive candidate strategies,
of which we have only begun to explore in this work.
For example, repeating the autoregression in Eq.~\ref{eqn:adaptive_sampling} could mimic the behavior of gradient descent,
yielding an exact local maximizer of the GP sample path.
This autoregression would scale poorly and require significant sequential computation, which may offer little benefit as our single-step adaptation performs well in practice.
Nevertheless, this extension and others are worth exploring, especially in unique high-dimensional application domains.

\subsubsection*{Acknowledgements}

This research was enabled in part by support provided by the Vector Institute, Advanced Research Computing at the University of British Columbia, and the Digital Research Alliance of Canada.
These resources were provided, in part, by the Province of Ontario, the Government of Canada through CIFAR, and companies sponsoring the Vector Institute.
We acknowledge the support of the Natural Sciences and Engineering Research Council and the Social Sciences and Humanities Research Council of Canada (NSERC: RGPIN-2024-06405, NFRFE-2024-00830).
GP is supported by the Canada CIFAR AI Chairs program.

{
  \small
  \bibliographystyle{plainnat}
  \bibliography{citations}
}
\section*{Checklist}

\begin{enumerate}

  \item For all models and algorithms presented, check if you include:
  \begin{enumerate}
    \item A clear description of the mathematical setting, assumptions, algorithm, and/or model. [Yes]
    \item An analysis of the properties and complexity (time, space, sample size) of any algorithm. [Yes]
    \item (Optional) Anonymized source code, with specification of all dependencies, including external libraries. [Yes]
  \end{enumerate}

  \item For any theoretical claim, check if you include:
  \begin{enumerate}
    \item Statements of the full set of assumptions of all theoretical results. [Yes]
    \item Complete proofs of all theoretical results. [Yes]
    \item Clear explanations of any assumptions. [Yes]
  \end{enumerate}

  \item For all figures and tables that present empirical results, check if you include:
  \begin{enumerate}
    \item The code, data, and instructions needed to reproduce the main experimental results (either in the supplemental material or as a URL). [Yes]
    \item All the training details (e.g., data splits, hyperparameters, how they were chosen). [Yes]
    \item A clear definition of the specific measure or statistics and error bars (e.g., with respect to the random seed after running experiments multiple times). [Yes]
    \item A description of the computing infrastructure used. (e.g., type of GPUs, internal cluster, or cloud provider). [Yes]
  \end{enumerate}

  \item If you are using existing assets (e.g., code, data, models) or curating/releasing new assets, check if you include:
  \begin{enumerate}
    \item Citations of the creator If your work uses existing assets. [Yes]
    \item The license information of the assets, if applicable. [Yes]
    \item New assets either in the supplemental material or as a URL, if applicable. [Yes]
    \item Information about consent from data providers/curators. [Not Applicable]
    \item Discussion of sensible content if applicable, e.g., personally identifiable information or offensive content. [Not Applicable]
  \end{enumerate}

  \item If you used crowdsourcing or conducted research with human subjects, check if you include:
  \begin{enumerate}
    \item The full text of instructions given to participants and screenshots. [Not Applicable]
    \item Descriptions of potential participant risks, with links to Institutional Review Board (IRB) approvals if applicable. [Not Applicable]
    \item The estimated hourly wage paid to participants and the total amount spent on participant compensation. [Not Applicable]
  \end{enumerate}

\end{enumerate}

\clearpage
\onecolumn
\appendix
\section{EXPERIMENTAL DETAILS}\label{appendix:experiments}
All our models and experiments were executed using the BoTorch library \cite{balandat2020botorch}. For the GP surrogate, we use a zero-mean function with the squared-exponential kernel, whose lengthscales are inferred from the data using Automatic Relevance Determination \cite{neal2012bayesian}.
Optimization of the lengthscale leverages a ``dimensional-scaled prior'' (DSP), where the lengthscale prior scales with the square-root of the dimensionality \cite{hvarfner2024vanilla}.
Following use of the DSP, the outputscale of this model is fixed at 1.
We also standardize the data during training to have zero mean and unit variance.
On benchmarks with $\geq 10000$ iterations we make use of GPyTorch's conjugate gradient-based inference \citep{gardner2018gpytorch};
otherwise we perform hyperparameter optimization and posterior inference without approximation.

The Log Expected Improvement, TuRBO, \textsc{Pathwise}, and SAASBO approaches use the default BoTorch implementations.
TuRBO hyperparameters include a default length of 0.8, a minimum side-length of $0.5^7$, a max length of $1.6$, and an adaptive failure tolerance of $\left\lceil \max\left\{\frac{4}{q}, \frac{d}{q}\right\}\right\rceil$ for batch size $q$.
When a restart is triggered for TuRBO, we include the 30 Sobol initial points in the evaluation budget.
For \textsc{Pathwise}, by default BoTorch uses 1024 features.
SAASBO uses the \textsc{qLogNoisyExpectedImprovement} acquisition function and a noise variance of $10^{-6}$.
Training the fully-Bayesian SAAS prior uses the NUTS sampler with $512$ warmup steps, $256$ samples, and a thinning factor of $16$.
For BAxUS, we use the original author implementation found at \url{https://github.com/LeoIV/BAxUS}.
For \textsc{Cylindrical TS}, we use the original author implementation found at \url{https://github.com/HW-AI-Research/CTS-HDBO}, using the same hyperparameters.
Specifically, we set $\sigma_\text{init} = 0.125$ and when not using TuRBO, we set $R = 2\sqrt{d}$.
We note that \textsc{Cylindrical TS} leverages local modelling by only performing hyperparameter optimization using observed points that live inside the spherical trust region, whereas the original TuRBO implementation always fits using all points possible (within each restart).

We run our experiments on a server with two Intel Xeon Silver 4116 CPUs with 192 GB of RAM and four NVIDIA Tesla V100 32GB GPUs.
However, when using \methodnameshort\ with $M=10^4$ points, we typically only use a single core and 8 GB of RAM.
A single optimization run on SVM with 1000 iterations typically requires 2 hours of wall time to complete for our method, RAASP, and BAxUS.
However, \textsc{LogEI} and \textsc{Pathwise} uses significantly more time due to gradient optimization steps, typically 6 or more hours.

\subsection{Wallclock Time}

\begin{table}[t]
  \centering
  \begin{tabular}{lllll}
    \toprule\noalign{}
    ~ & Rover (60D) & MOPTA08 (124D) & LassoDNA (180D) & SVM (388D) \\
    \midrule\noalign{}
    ACTS & 0.53 ± 0.14 & 0.52 ± 0.10 & 0.53 ± 0.07 & 0.65 ± 0.12 \\
    RAASP & 0.44 ± 0.06 & 0.45 ± 0.06 & 0.47 ± 0.05 & 0.53 ± 0.06 \\
    Cylindrical & 0.75 ± 0.06 & 0.77 ± 0.06 & 0.80 ± 0.07 & 0.92 ± 0.07 \\
    Pathwise & 1.57 ± 0.53 & 4.61 ± 2.24 & 6.09 ± 2.33 & 11.07 ± 2.93 \\
    \bottomrule\noalign{}
  \end{tabular}
  \caption{Wallclock time (in seconds $\pm$ 2 std. err.) to propose a single candidate of selected TS methods.}
  \label{tab:wallclock_time}
\end{table}

In Table \ref{tab:wallclock_time}, we report the wallclock time of mean acqusition time over 10 runs for selected TS methods.
The fastest method is RAASP, followed by \methodnameshort, then \textsc{Cylindrical TS}, and finally \textsc{Pathwise}.
ACTS incurs a minor overhead from sampling the gradient, where as \textsc{Cylindrical TS} requires additional time to sample from a truncated Gaussian distribution.
All three approaches are limited by the Cholesky decomposition, which is performed on $10^4$ points.
Lastly, \textsc{Pathwise} requires significantly more time due to the need to optimize the acquisition function using gradient-based optimization.

\subsubsection{Memory Usage}
\begin{table}[t]
  \centering
  \begin{tabular}{  lllll}
    \toprule\noalign{}
    ~ & Rover (60D) & MOPTA08 (124D) & LassoDNA (180D) & SVM (388D) \\
    \midrule\noalign{}
    ACTS & 4.64 ± 0.30 & 4.58 ± 0.18 & 4.57 ± 0.06 & 4.73 ± 0.23 \\
    RAASP & 3.93 ± 0.10 & 3.95 ± 0.10 & 3.96 ± 0.10 & 4.01 ± 0.10 \\
    Cylindrical & 4.77 ± 0.10 & 4.79 ± 0.10 & 4.82 ± 0.11 & 4.93 ± 0.13 \\
    Pathwise & 0.08 ± 0.04 & 0.12 ± 0.06 & 0.17 ± 0.08 & 0.36 ± 0.14 \\
    \bottomrule\noalign{}
  \end{tabular}
  \caption{Mean peak memory usage (in GiB $\pm$ 2 std. err.) of selected TS methods.}
  \label{tab:memory_usage}
\end{table}

In Table \ref{tab:memory_usage}, we measure the peak memory usage across 10 runs for selected TS methods as reported by PyTorch's \texttt{cuda.max\_memory\_allocated} in GiB.
Similarly to the wall-clock time, the chief space usage comes from the Cholesky decomposition, which is performed on $10^4$ points for all methods except \textsc{Pathwise}, which only evaluates the posterior for random restarts on 512 points for BoTorch's \texttt{sample\_around\_best} optimization initializer.

\section{EXTENDED OPTIMIZATION RESULTS}\label{appendix:seq_opt}
In Figs.~\ref{fig:extended_opt_global} and \ref{fig:extended_opt_turbo}, we provide the full results of the high-dimensional optimization problems.
We compare to SAASBO \cite{eriksson2021high}, which uses a fully-Bayesian surrogate model to identify sparse subspaces for optimization, BAxUS \cite{papenmeier2022increasing}, a latent-space approach that progressively increases effective search dimension, and the Log Expected Improvement in conjunction with the dimensional-scaled prior \cite{ament2023unexpected, hvarfner2024vanilla}.
We observe that \methodnameshort\ exhibits strong optimization performance in several objectives,
matching or exceeding performance on many benchmarks.
The notable exceptions are the SVM and Median Molecules 2 benchmarks, where \methodnameshort\ is outperformed by LogEI and BAxUS.
We do note however that \methodnameshort\ achieves much higher objective values than BAxUS on all other benchmarks.

We additionally perform a statistical analysis using the Friedman and Nemenyi tests and aggregate rankings across all sequential optimization benchmarks, summarized in Figure \ref{fig:ranking_significance}.
\methodnameshort\ achieves the highest mean rank among all TS strategies.
It is statistically distinguishable ($p < 0.05$) from \textsc{Pathwise} and \textsc{Cylindrical TS} in the global setting.
Even within the competitive TuRBO setting, \methodnameshort\ retains the top rank.

\begin{figure}[t]
  \centering`'
  \includegraphics[width=\linewidth]{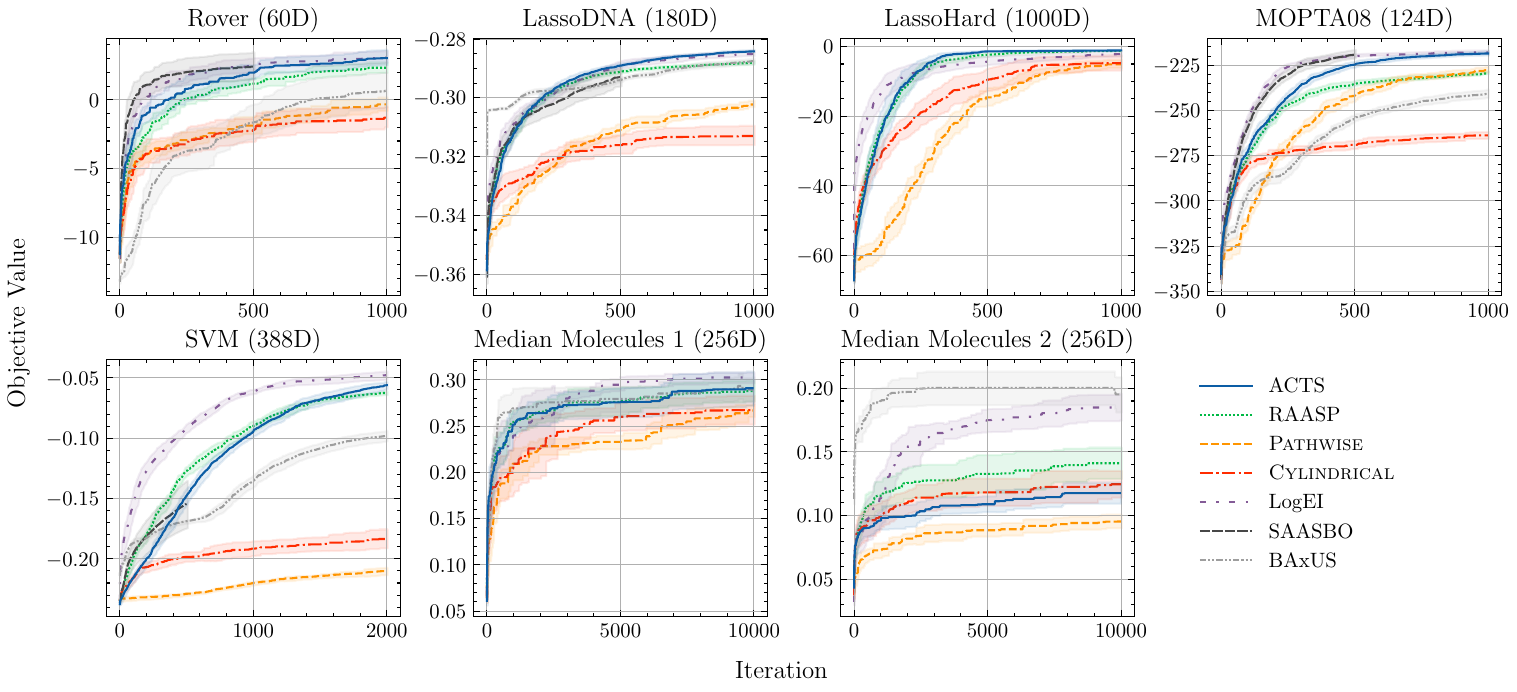}
  \caption{{\bf Optimization performance of high-dimensional optimization problems without TuRBO}. The TS methods exhibit strong performance compared to several non-TS benchmarks,
  though notably the LogEI and BAxUS baselines obtain significantly better performance on Median Molecules 2.}
  \label{fig:extended_opt_global}
\end{figure}

\begin{figure}[!ht]
  \centering
  \includegraphics[width=\linewidth]{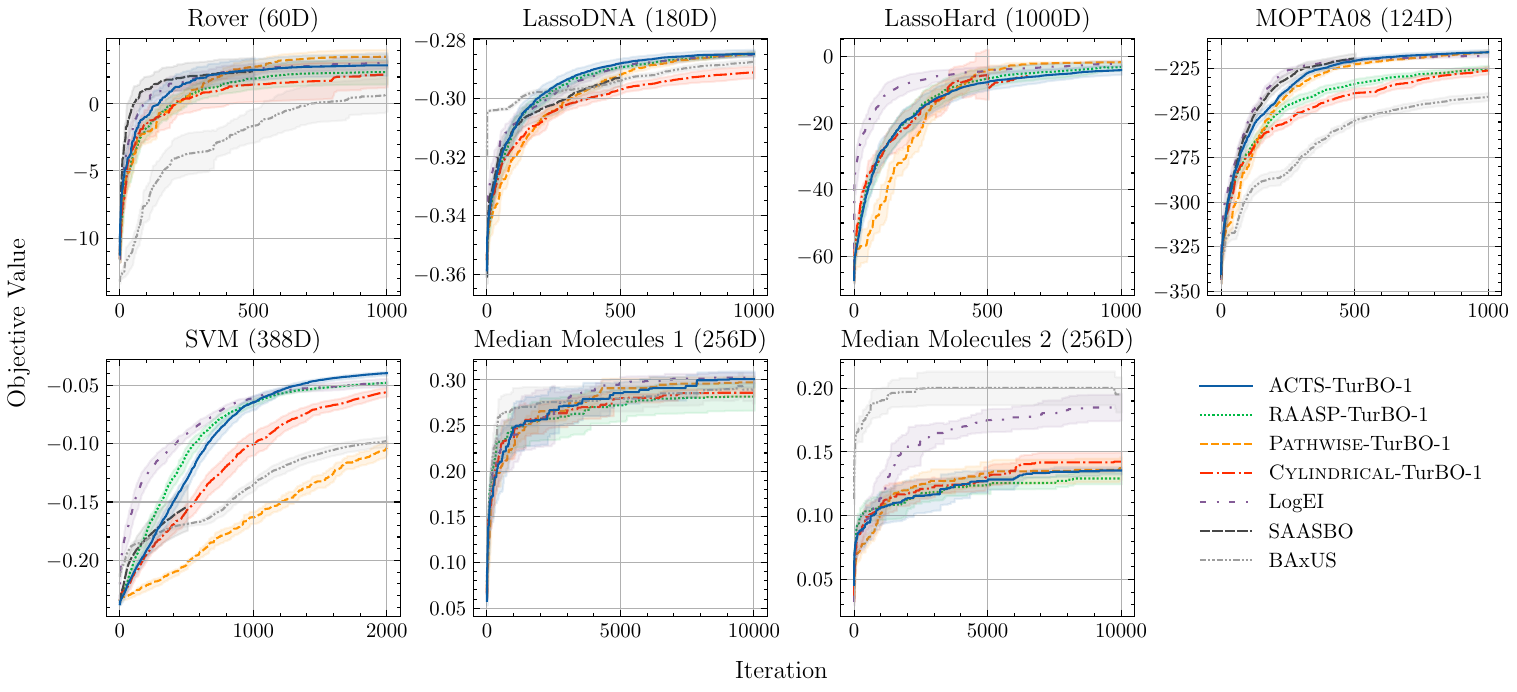}
  \caption{{\bf Optimization performance of high-dimensional optimization problems with TuRBO}. When using TuRBO trust regions, the Thompson sampling methods' optimization performance can be enhanced.}
  \label{fig:extended_opt_turbo}
\end{figure}

\begin{figure}
  \centering
  \includegraphics[width=\linewidth]{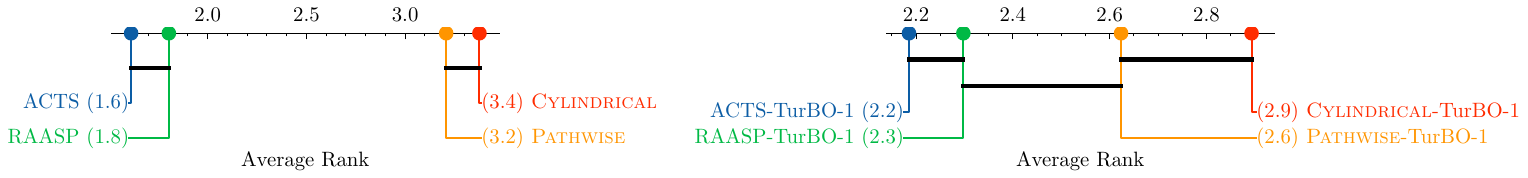}
  \caption{Ranking significance of selected TS methods. \textbf{Left:} Global setting. \textbf{Right:} TuRBO setting.}
  \label{fig:ranking_significance}
\end{figure}

\begin{figure}[!ht]
  \centering
  \includegraphics[width=\linewidth]{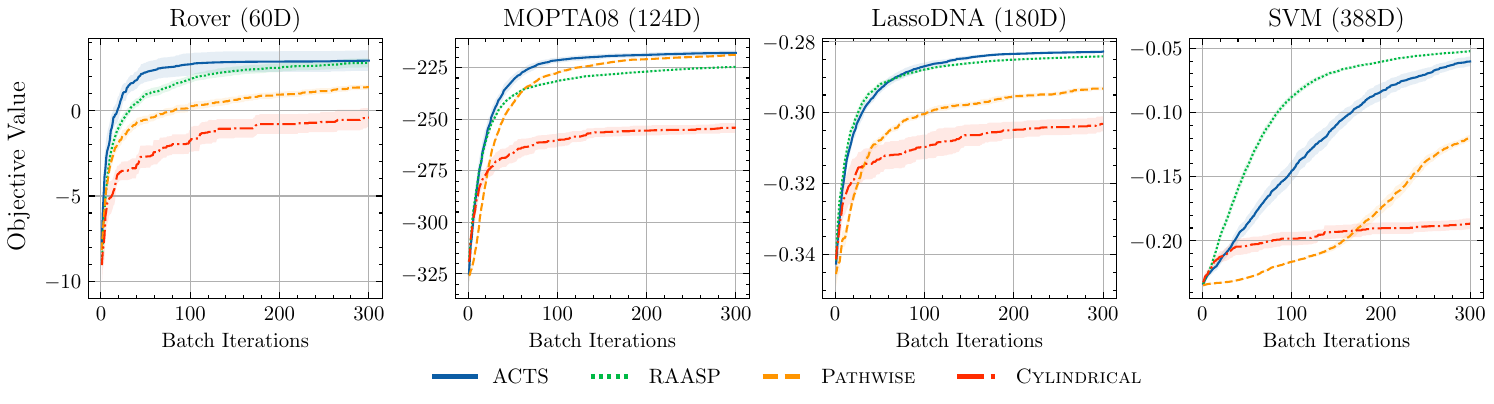}
  \caption{{\bf Optimization performance of selected high-dimensional benchmarks of small batches ($q=10$).} \methodnameshort\ achieves strong performance in all benchmarks compared to other TS methods, though potentially requiring a larger evaluation budget to overcome RAASP in SVM.}
  \label{fig:batch_opt_10}
\end{figure}

\begin{figure}[!ht]
  \centering
  \includegraphics[width=\linewidth]{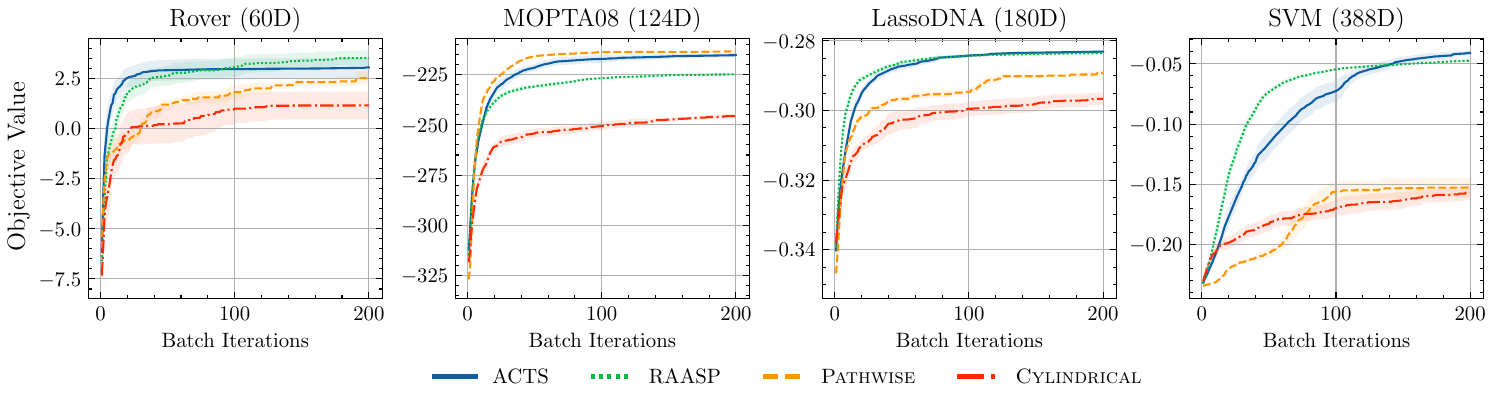}
  \caption{{\bf Optimization performance of selected high-dimensional benchmarks of medium-sized batches ($q=50$).} \methodnameshort\ ranks high among TS methods, with \textsc{Pathwise} finding the best performance in MOPTA08.}
  \label{fig:batch_opt_50}
\end{figure}

\section{EXTENDED BATCH OPTIMIZATION RESULTS}\label{appendix:batch_opt}
Here, we complement our earlier $q=100$ batch optimization setting with $q=10$ (Figure \ref{fig:batch_opt_10}) and $q=50$ (Figure \ref{fig:batch_opt_50}).
The strong optimization performance of \methodnameshort\ is maintained across the batch settings, only losing to RAASP in the short term (we see in $q=50$ \methodnameshort\ eventually overtakes RAASP).
\textsc{Pathwise} exhibits varying performance while \textsc{Cylindrical TS} ranks poorly.

\paragraph{Batch Diversity.} When constructing a batch, it is typically important to have a diverse set of points so that posterior uncertainty is reduced in many locations.
That is, no two points in the batch should be too close nor all points in the batch be too far apart.
Our approach naturally balances this tradeoff.
At the start of optimization, when few points are known, $\{\nabla f^{(k)}\}_{k=1}^q$ are likely to be dissimilar, inducing different search spaces for each point in the batch.
Uncertainty in our gradient samples decrease as more points near $\vx_0$ are explored, and thus points in the batch will be more similar.
However, randomness in our candidate points and through sampling inherently affords us some diversity even when exploiting near local maxima.

\section{GLOBAL CONSISTENCY OF \methodnameshort}
\label{appendix:consistency}

We show that \methodnameshort\ will query the global maximizer as the number of evaluations tends to infinity.

\begin{reptheorem}{thm:global_consistency}
  Choose any $\epsilon > 0$ and assume the following:
  \begin{quote}
    \begin{enumerate}[label=A\arabic*.]
      \item $k$ is a non-degenerate kernel with $k(\vx, \vx) = 1$ for all $\vx \in \inputspace$ (standard stationary kernel with no outputscale)
      \item $k$, $\nabla k$, and $k\nabla^\top$ are bounded.
      \item $|F(\vx)| \leq B$, $\forall \vx \in \mathcal{X}$ (bounded objective function)
      \item $y_t = F(\vx_t) + \eta$, $|\eta| \leq b$ (bounded observation noise)
      \item $\mathcal{X} \subset \R^d$ is compact with $d < \infty$
      \item The candidate set $\widetilde{\mathcal{X}}_t$ is constructed by sampling $M$ points independently from some fully-supported distribution over $\mathcal{T}_{\nabla f(\vx_0)} \cap \inputspace_\epsilon$,
        where $\inputspace_\epsilon$ is some $\epsilon$-covering of $\inputspace$ (with respect to $\Vert \cdot \Vert_2$).
    \end{enumerate}
  \end{quote}
  Then running Bayesian optimization with \methodnameshort\ will, in finite time, almost surely query some point $\vx$ such that $\Vert \vx - \vx^* \Vert_2 \leq \epsilon$,
  where $\vx^* = \argmax_{x \in \mathcal{X}} F(\vx)$ is the global maximizer of $F$.
\end{reptheorem}

We decompose this proof into several lemmas and delay presentation to the end.
The general outline of our proof is to use contradiction to define an event where some point $\vx_u$ has not been observed, and show that the event that the \methodnameshort\ search space will occur infinitely often.
We note that, by A6, ACTS will only ever sample candidate points from $\inputspace_\epsilon$,
and thus all observations $\{ \vx_i \}_{i=1}^\infty$ will be from the finite set $\inputspace_\epsilon$.

The first lemma is a generalized concentration inequality that we will use to upper bound the posterior maximizer over subsets of $\inputspace_\epsilon - \{ \vx_u \}$.

\begin{lemma}
  \label{lma:mvn_cdf_bound}
  Fix $p \in \mathbb N$.
  For any constant $\mu_0 \in \R$, any (strictly) positive definite matrix $\mS_0 \in \R^{p \times p}$,
  and any $\alpha > \sqrt p \Vert \mS_0 \Vert_2 + \mu_0$,
  there exists some constant $C_0 > 0$ that depends only on $p$, $\mu_0$, $\mS_0$, and $\alpha$ such that
  $$ \Pr_{\vz \sim \normal \left( \vm, \mS \right)} \left( \Vert \vz \Vert_\infty < \alpha \right)
  > C_0 $$
  for all $\vm \in \R^p$ with $\Vert \vm \Vert_2 \leq \mu_0$ and all $\mS$ with
  $\vzero \prec \mS \preceq \mS_0$.
\end{lemma}

\begin{proof}
  We have that:
  \begin{align*}
    \Pr_{\vz \sim \normal \left( \vm, \mS \right)} \left( \Vert \vz \Vert_\infty < \alpha \right)
    &= 1 \:\: - \Pr_{\vz \sim \normal \left( \vm, \mS \right)} \left( \Vert \vz \Vert_\infty \geq \alpha \right)
    \\
    &= 1 \:\: - \Pr_{\vz' \sim \normal \left( \vzero, \mI \right)} \left( \Vert \mS^{1/2} \vz' + \vm \Vert_\infty \geq \alpha \right)
  \end{align*}
  We have that the function $g(\vz') := \Vert \mS^{1/2} \vz' + \vm \Vert_\infty$ is Lipschitz continuous with respect to the Euclidean norm
  (as it is the composition of Lipschitz continuous functions) with Lipschitz constant $\sqrt{\Vert \mS \Vert_2}$.
  By concentration of measure for Lipschitz functions of Gaussian random variables \citep[][Thm.~2.26]{wainwright2019high},
  for any $t > 0$ we have that
  \begin{align*}
    \Pr_{\vz \sim \normal \left( \vm, \mS \right)} \left( \Vert \vz \Vert_\infty - \E[\Vert \vz \Vert_\infty] \geq t \right)
    \leq \exp\left( -\frac{t^2}{2 \Vert \mS \Vert_2} \right)
    \leq \underbrace{\exp\left( -\frac{t^2}{2 \Vert \mS_0 \Vert_2} \right)}_{:= C_0} < 1.,
  \end{align*}
  where the last inequality follows from the fact that $\Vert \mS \Vert_2 \leq \Vert \mS_0 \Vert_2$.
  Now note that:
  \begin{align*}
    \E[ \Vert \vz \Vert_\infty ] \leq \E[ \Vert \vz \Vert_2 ]
    &\leq \E[ \Vert \vz - \vm \Vert_2] + \Vert \vm \Vert_2
    \\
    & \leq \E_{z' \sim \normal(\vzero, \mI)}[ \Vert \mS \vz' \Vert_2] + \Vert \vm \Vert_2
    \\
    & \leq \Vert \mS \Vert_2 \E_{z' \sim \normal(\vzero, \mI)}[ \Vert \vz' \Vert_2] + \Vert \vm \Vert_2
    \\
    & \leq \Vert \mS \Vert_2 \sqrt{\E_{z' \sim \normal(\vzero, \mI)}[ \Vert \vz' \Vert_2^2]} + \Vert \vm \Vert_2
    \tag{Jensen's inequality}
    \\
    & \leq \sqrt p \Vert \mS \Vert_2 + \Vert \vm \Vert_2 \leq \sqrt p \Vert \mS_0 \Vert_2 + \mu_0.
  \end{align*}
  Setting $t = \alpha - \left( \sqrt p \Vert \mS_0 \Vert_2 + \mu_0 \right)$ gives the desired result.
\end{proof}

We now use this lemma to bound the posterior maximizer over any finite subset of $\inputspace_\epsilon$.

\begin{lemma}
  \label{lma:noisy_observation_bound}
  Assume A1-A6.
  For any $\mathcal S \subseteq \inputspace_\epsilon$,
  there exists some constants $C_\mathcal{S} > 0$, $U_S < \infty$
  such that, for all $t > 0$ and $\mathcal D_t = \{(\vx_i, y_i)\}_{i=1}^t \in (\inputspace_\epsilon \times \R)^t$,
  $$ \Pr\left( \Vert \vf_{\mathcal S} \Vert_\infty < U_S \: \mid \: \mathcal D_t \right) > C_{\mathcal S}. $$
\end{lemma}

\begin{proof}[Proof of Lemma~\ref{lma:noisy_observation_bound}]
  Define $\vf_{\mathcal S} := [f(\vx)]_{\vx \in \mathcal S}$ as the vector of (noiseless and unobserved) objective function values at all points in $\mathcal S$.
  The posterior distribution $f_{\mathcal S} \mid \mathcal D_t$ is a multivariate normal random variable,
  with mean $\vmu$ and covariance $\mSigma$ given by the standard GP posterior formulas:
  \begin{gather*}
    \vf_{\mathcal S} \mid \mathcal D_t \:\: \sim \:\: \normal\left( \vmu, \mSigma \right) \\
    \vmu := \mK_{\mathcal S, \mathcal O} \left( \mK_{\mathcal O} + \mD \right)^{-1} \bar \vy_{\mathcal O}, \qquad
    \mSigma := \mK_{\mathcal S} - \mK_{\mathcal S, \mathcal O} \left( \mK_{\mathcal O} + \mD \right)^{-1} \mK_{\mathcal S, \mathcal O}^\top,
  \end{gather*}
  where $\mathcal O := \cup_{i=1}^t \{ \vx_i \} \subseteq \inputspace_\epsilon$ is the set of all inputs that have been observed at least once;
  $\mK_{\mathcal O} \in \R^{\vert \mathcal O \vert \times \vert \mathcal O \vert}$ and $\mK_{\mathcal O} \in \R^{\vert \mathcal S \vert \times \vert \mathcal S \vert}$ are the Gram matrices of all points in $\mathcal O$ and $\mathcal S$ respectively;
  $\mK_{\mathcal S, \mathcal O} \in \R^{\vert \mathcal S \vert \times \vert \mathcal O \vert}$ is the cross-Gram matrix between $\mathcal S$ and $\mathcal O$;
  $\mD \in \R^{\vert \mathcal O \vert \times \vert \mathcal O \vert}$ is a diagonal matrix with $D_{ii} = \sigma^2 / n_i$ where $n_i$ is the number of times $\vx_i$ has been observed;
  and $\bar \vy_{\mathcal O} \in \R^{\vert \mathcal O \vert}$ is the vector of average observed values at each input in $\mathcal O$.

  Letting $\mu(\cdot): \inputspace_\epsilon \to \R$ be the posterior mean function of $f$ and letting $\rkhs$ be the reproducing kernel Hilbert space defined by $k(\cdot, \cdot)$,
  we first note that
  $$
  \begin{aligned}
    [\vmu]_i &= \left\langle k(\vx_{\mathcal S_i}, \cdot), \mu(\cdot) \right\rangle_{\rkhs}
    \\
    &\leq \underbrace{\Vert k(\vx_{\mathcal S_i}, \cdot) \Vert_{\rkhs} }_{1} \underbrace{\Vert \mu(\cdot) \Vert_{\rkhs}}_{
      \sqrt{ \vy_{\mathcal O}^\top \left( \mK_{\mathcal O} + \mD \right)^{-1} \mK_{\mathcal O} \left( \mK_{\mathcal O} + \mD \right)^{-1} \vy_{\mathcal O} }
    }
    \\
    &\leq \sqrt{ \vy_{\mathcal O}^\top \left( \mK_{\mathcal O} \right)^{-1} \vy_{\mathcal O} }
    \\
    &\leq \sqrt{ \tfrac{1}{\lambda_{\min}\left( \mK_{\mathcal O} \right)} } \Vert \bar \vy_{\mathcal O} \Vert_2
    \\
    &\leq \sqrt{ \tfrac{1}{\lambda_{\min}\left( \mK_{\mathcal O} \right)} } (B + b) \sqrt{\vert \mathcal O \vert}
    \\
    &\leq \underbrace{\min_{\mathcal C \in 2^\inputspace_\epsilon} \sqrt{ \tfrac{1}{\lambda_{\min}\left( \mK_{\mathcal C} \right)} }}_{=: C_0 < \infty} (B + b) \sqrt{\vert \mathcal X_\epsilon \vert},
  \end{aligned}
  $$
  where $\lambda_{\min}(\cdot)$ denotes the minimum eigenvalue of a matrix.
  The finiteness of $C_0$ follows from the fact that $\inputspace_\epsilon$ is finite (and thus all $\mK_{\mathcal C}$  are finite matrices) and $k$ is a non-degenerate kernel.
  Therefore, $\Vert \vmu \Vert_2 \leq C_0 (B + b) \sqrt{\vert \mathcal X_\epsilon \vert} \sqrt{\vert \mathcal S \vert} < C_0 (B + b) \vert \mathcal X_\epsilon \vert$.
  Moreover, we have that
  $$ \vzero
  \:\: \prec \:\: \mSigma \:\: \preceq \:\:
  \underbrace{
    \mK_{\mathcal S}
  }_{=: \mSigma_0}
  $$
  Because $\mathcal S$ is finite and $k$ is bounded on $\inputspace \times \inputspace$,
  we have that $\Vert \mSigma_0 \Vert_2 < \infty$.
  Defining $U_{\mathcal S}$ to be any constant such that
  $$ C_0 (B + b) \vert \mathcal X_\epsilon \vert \: < \: U_{\mathcal S} \: < \: \infty, $$
  by Lemma~\ref{lma:mvn_cdf_bound}, there exists some constant $C_{\mathcal S} > 0$ that depends only on
  $\vert \mathcal S \vert$, $B$, $b$, $U_{\mathcal S}$, and $\mSigma_0$
  such that
  $$ \Pr\left( \Vert \vf_{\mathcal S} \Vert_\infty < U_{\mathcal S} \: \mid \: \mathcal D_t\right)
  \:\: = \:\: \Pr_{\mathbf z \sim \normal(\vmu, \mSigma)} \left( \Vert \vz \Vert_\infty < U_{\mathcal S} \right)
  \:\: > \:\: C_{\mathcal S}. $$
  Note that this bound does not depend on $t$ or $\mathcal D_t$, and thus the result follows.
\end{proof}

The next two lemmas establish a lower bound for posterior samples at any unobserved point cointained in the ACTS search space.

\begin{lemma}
  For any $\vx_0, \vx_u \in \inputspace$,
  let $\mathcal{C}(\vx_0, \vx_u) = \{\nabla f(\vx_0) \in \mathbb{R}^d \mid \vx_u \in \mathcal{T}_{\nabla f(\vx_0)}\}$ be the set of gradients for $\vx_0$
  that generate ACTS search spaces containing $\vx_u$.
  Then $\mathcal{C}(\vx_0, \vx_u)$ has strictly positive Lebesgue measure.
  \label{lem:lebesgue_measure}
\end{lemma}
\begin{proof}
  We can re-write $\mathcal{C}(\vx_0, \vx_u)$ as the following:
  \begin{equation*}
    \mathcal{C}(\vx_0, \vx_u)  = \left\{\vv \in \mathbb{R}^d \mid v_i > 0 \text{ if } [\vx_u]_i - [\vx_0]_i > 0, v_i < 0 \text{ if } [\vx_u]_i - [\vx_0]_i < 0, v_i \in \mathbb{R} \text{ if } [\vx_u]_i - [\vx_0]_i = 0 \right\}.
  \end{equation*}
  Consider the indicator vector $\vw \in \R^d$ where
  \begin{equation*}
    w_i =
    \begin{cases}
      1 & [\vx_u]_i - [\vx_0]_i > 0 \\
      -1 & [\vx_u]_i - [\vx_0]_i < 0 \\
      1 &  [\vx_u]_i - [\vx_0]_i = 0,
    \end{cases}
  \end{equation*}
  where in the last case we have arbitrarily chosen $w_i = 1$.
  We note that $\vw$ trivially satisfies the conditions to be an element of $\mathcal{C}(\vx_0, \vx_u)$.
  Now consider any $\vv' \in \mathcal{B}(\vw, 1/2)$,
  where $\mathcal{B}(\vw, 1/2)$ denotes the open ball of radius $1/2$ centered at $\vw$,
  such that $\vert v_i' - w_i \vert \leq \Vert \vv' - \vw \Vert < 1/2$ for all $i \in [1, d]$.
  If $[\vx_u]_i - [\vx_0]_i > 0$, then $w_i = 1$ and thus $v'_i \in (0.5, 1.5)$,
  thus satisfying the condition $v_i' > 0$ that is required for $\vv' \in \mathcal{C}(\vx_0, \vx_u)$.
  The other cases can be verified similarly.
  Because this holds for all $i \in [1, d]$, we have that all $\vv' \in \mathcal{B}(\vw, 1/2)$ are elements of $\mathcal{C}(\vx_0, \vx_u)$.
  Thus $\mathcal{C}(\vx_0, \vx_u)$ contains a non-empty open-set and so it has strictly positive Lebesgue measure.
\end{proof}

\begin{lemma}
  \label{lem:joint_lower_bound}
  Assume A1-A6.
  Let $\vx_u \in \inputspace_\epsilon$ be a fixed point that \methodnameshort{} has not observed,
  and define $\mathcal S$ as in Lemma~\ref{lma:noisy_observation_bound}.
  Given the incumbent $\vx_0$,
  define
  $\mathcal{C}(\vx_0, \vx_u) = \{\nabla f(\vx_0) \in \mathbb{R}^d \mid \vx_u \in \mathcal{T}_{\nabla f(\vx_0)}\}$
  as in Lemma~\ref{lem:lebesgue_measure}.
  Given some $U < \infty$, there exists some constant $C_U$ such that
  \begin{equation*}
    \Pr\left(\nabla f(\vx_0) \in \mathcal{C}(\vx_0, \vx_u), \: f(\vx_u) > U \: \mid \: \vf_\mathcal{S}\right) \geq C_U > 0
  \end{equation*}
  for all $\Vert \vf_{\mathcal S} \Vert_\infty \leq U$.
\end{lemma}
\begin{proof}
  From the Jacobian GP model, we have the following posterior distributions:
  \begin{equation*}
    \begin{bmatrix}
      f(\vx_u) \\ \nabla f(\vx_0)
    \end{bmatrix} \: \mid \: \vf_\mathcal{S} \:
    \sim \mathcal{N}(\vmu_t, \mSigma_t)
  \end{equation*}
  \begin{equation*}
    \vmu =
    \begin{bmatrix}k(\vx_u, \mathcal{S}) \\ \nabla k(\vx_0, \mathcal{S})
    \end{bmatrix}k(\mathcal{S}, \mathcal{S})^{-1}\vf_\mathcal{S}
  \end{equation*}
  \begin{equation*}
    \mSigma =
    \begin{bmatrix}
      k(\vx_u, \vx_u) & k(\vx_u, \vx_0)\nabla^\top \\
      \nabla k(\vx_0, \vx_u) & \nabla k(\vx_0, \vx_0)\nabla^\top
    \end{bmatrix}
    -
    \begin{bmatrix}k(\vx_u, \mathcal{S}) \\ \nabla k(\vx_0, \mathcal{S})
    \end{bmatrix}k(\mathcal{S}, \mathcal{S})^{-1}
    \begin{bmatrix}k(\vx_u, \mathcal{S}) \\ \nabla k(\vx_0, \mathcal{S})
    \end{bmatrix}^\top
  \end{equation*}
  Note that
  \begin{equation*}
    ||k(\mathcal{S}, \mathcal{S})^{-1}|| = \frac{1}{\lambda_\text{min}(K_{\mathcal S})} < \infty,
  \end{equation*}
  where $\lambda_\text{min}$ denotes the minimum eigenvalue.
  (Finiteness follows from the fact that $\mathcal S$ is finite and $k$ is a non-degenerate kernel.)
  Thus, $\vmu$ and $\mSigma$ are all bounded quantities due to the boundedness of $k$ and $\nabla k$.
  We write the desired probability as the integral
  \begin{equation*}
    \int_{\vg \in \mathcal{C}}\int_{y\in{(U, \infty)}}p\left(\nabla f(\vx_0) = \vg, f(\vx_u)=y \mid \vf_\mathcal{S}\right)d\vg\,dy
  \end{equation*}
  Given any $\vg$, $y$, and $\vf_\mathcal{S}$, the density on the inside of the integral is Gaussian and thus strictly positive.
  By Lemma 3, the integration sets have strictly positive measure, and thus the integral is strictly positive for any choice of $U$ and $\vf_{\mathcal S}$.
  Finally, for any fixed $U$, we note that this integral is continuous with respect to $\vf_{\mathcal S}$,
  and thus by the extreme value theorem, it attains its minimum $C_u$ over the compact set $\{\vf_{\mathcal S} \mid \Vert \vf_{\mathcal S} \Vert_\infty \leq U\}$.
  Therefore, this integral is bounded below by $C_u > 0$.
\end{proof}

Now we are ready to prove Theorem \ref{thm:global_consistency}.
\begin{proof}[Proof of Theorem \ref{thm:global_consistency}]
  Assume to the contrary that there exists some $\vx_u$ that will never be observed (i.e., $F(\vx_u)$ is never evaluated).
  Fix $\mathcal S = \inputspace_\epsilon - \{ \vx_u \}$.
  Consider the event $E_t$ defined as:
  \begin{equation*}
    E_t = (f(\vx_u) > U_\mathcal{S}, \vx_u \in \tilde{\mathcal{X}}_t, ||\vf_\mathcal{S}||_\infty \leq U_\mathcal{S}) \mid \mathcal{D}_t.
  \end{equation*}
  where $\vx_0$ is the incumbent at iteration $t$
  and $U_\mathcal{S} < \infty$ is a constant to be determined later.
  Intuitively, this is an event where $\vx_u$ is a candidate point and the posterior maximizer at iteration $t$.
  We will express the probability of $E_t$ as an integral over the joint density
  \begin{align*}
    p\left(\mathbf{1}[\vx_u \in \tilde{\mathcal{X}}_t], \nabla f(\vx_0), f(\vx_u), \vf_\mathcal{S} \mid \mathcal{D}_t \right)
    &=
    p\left(\mathbf{1}[\vx_u \in \tilde{\mathcal{X}}_t] \mid \nabla f(\vx_0)\right)
    p\left(\nabla f(\vx_0), f(\vx_u) \mid \vf_\mathcal{S}, \mathcal{D}_t  \right)
    p\left(\vf_\mathcal{S} \mid \mathcal{D}_t \right)
  \end{align*}
  where the first factor is conditionally independent of $\mathcal{D}_t$ and $f(\vx_u)$ (as $\tilde{\mathcal{X}}_t$ is constructed by sampling $M$ points from $\mathcal{T}_{\nabla f(\vx_0)}$),
  and the second factor is conditionally independent of $\mathcal{D}_t$ given that $\{ x_i \}_{i=1}^t \subseteq \mathcal S$.
  The probability of $E_t$ can be thus be decomposed as
  \begin{align*}
    P(E_t) &=
    \int_{\vg \in \mathcal{C}(\vx_0, \vx_u)}
    P\left(\vx_u \in \tilde{\mathcal{X}}_t \mid \nabla f(\vx_0) = \vg \right)
    \int_{\vy' \in \mathcal{U}}
    p\left(\vf_\mathcal{S} = \vy' \mid \mathcal{D}_t \right)
    \int_{y \in (U, \infty)}
    p\left(\nabla f(\vx_0) = \vg, f(\vx_u) = y \mid \vf_\mathcal{S} = \vy' \right)
    \\
    & \qquad \qquad \qquad \times d\vg\,d\vy'\,dy,
  \end{align*}
  where $\mathcal{C}(\vx_0, \vx_u) = \{\nabla f(\vx_0) \in \mathbb{R}^d \mid \vx_u \in \mathcal{T}_{\nabla f(\vx_0)}\}$ as in Lemma \ref{lem:lebesgue_measure}
  and $\mathcal{U} := \{\vy \mid ||\vy||_\infty < U\}$.
  First, as $\vg \in \mathcal{C}(\vx_0, \vx_u)$ implies $\vx_u \in \mathcal{T}_{\nabla f(\vx_0)}$, we note that
  $$
  P\left(\vx_u \in \tilde{\mathcal{X}}_t \mid \nabla f(\vx_0) = \vg \right) \geq
  \underbrace{ \min_{\mathcal{T}_{\nabla f(\vx_0)}} \min_{\vx_u \in \mathcal{T}_{\nabla f(\vx_0)}} P\left( \vx_u \in \tilde{\mathcal{X}}_t \mid \mathcal{T}_{\nabla f(\vx_0)} \right)}_{=: C_1}
  > 0
  \qquad \forall \vg \in \mathcal{C}(\vx_0, \vx_u),
  $$
  where the first inequality uses the fact that the set of all possible $\mathcal{T}_{\nabla f(\vx_0)}$ is finite
  (as $\vx_0 \in \inputspace_\epsilon$, $\inputspace_\epsilon$ is finite, and there are finitely-many axis-aligned cones for each $\vx_0$)
  and the last inequality is due to A6.
  Thus, after applying Fubini's theorem we have that:
  \begin{align*}
    P(E_t) &\geq C_1
    \int_{\vy' \in \mathcal{U}}
    p\left(\vf_\mathcal{S} = \vy' \mid \mathcal{D}_t \right)
    \int_{\vg \in \mathcal{C}(\vx_0, \vx_u)}
    \int_{y \in (U, \infty)}
    p\left(\nabla f(\vx_0) = \vg, f(\vx_u) = y \mid \vf_\mathcal{S} = \vy' \right)
    \times dy\,d\vg\,d\vy'.
  \end{align*}
  Next, the innermost double integral is lower bounded by some constant $C_2 > 0$ from Lemma \ref{lem:joint_lower_bound}.
  Finally, we have
  \begin{align*}
    P(E_t) &> C_1 C_2 \int_{\vg \in \mathcal{C}(\vx_0, \vx_u)} \int_{\vy' \in \mathcal{U}} p(\vf_\mathcal{S} = \vy' \mid \mathcal{D}_t ) d\vy' d\vg
    \\
    &=  C_1 C_2
    P(\Vert \vf_\mathcal{S} \Vert_\infty < U \mid \mathcal{D}_t)
    \\
    &=  C_1 C_2 C_3,
  \end{align*}
  where $C_3 > 0$ is the constant denoted as $C_{\mathcal S}$ in Lemma~\ref{lma:noisy_observation_bound}
  (where we now set $U_\mathcal{S}$ to be the corresponding $U_\mathcal{S}$ from that lemma).
  We note that these constants are all strictly positive and do not depend on $t$.
  Thus, by the counterpart of the Borel-Cantelli lemma, we have that $E_t$ occurs infinitely often, and thus $\vx_u$ will be selected as the Thompson sample maximizer with probability 1 as $t \to \infty$, contradicting our previous assumption.
\end{proof}

\section{CURSE OF DIMENSIONALITY OF STANDARD THOMPSON SAMPLING}\label{appendix:curse_dim}
We illustrate the poor approach of ``Standard Thompson Sampling'' in high-dimensional problems, where such an approach involves using a space-filling sequence to construct the candidate points $\tilde{\mathcal{X}}$. Then $\vx_{t+1} = \argmax_{\vx \in \tilde{\mathcal{X}}} f(x)$, where $f \sim p(f\mid \mathcal{D}_t)$. We recall $10^4$ is a typical value for $|\tilde{\mathcal{X}}|$.

We highlight this by giving an evaluation budget of $10^6$ points. In Figure \ref{fig:curse_dimensionality}, we ask: if we are restricted to only the points given by the $\pi_\text{Sobol}$ or $\pi_\text{Uniform}$ candidate policies, how far from the optimal will our Thompson sampling be restricted to?
For small dimensional problems like Hartmann (6D), $10^4$ sufficiently provides enough points to have a small amount of regret.
However, for higher-dimensional problems, even $10^6$ points are insufficient.
In other words, given a candidate point budget of $10^4$ and an identical batch budget, even after 100 iterations, $\pi_\text{Sobol}$ or $\pi_\text{Uniform}$ fail to find good solutions.
As many standard Thompson sampling approaches use one of these approaches, they are inherently limited in high-dimensional problems.

\begin{figure}[t]
  \centering
  \includegraphics[width=\linewidth]{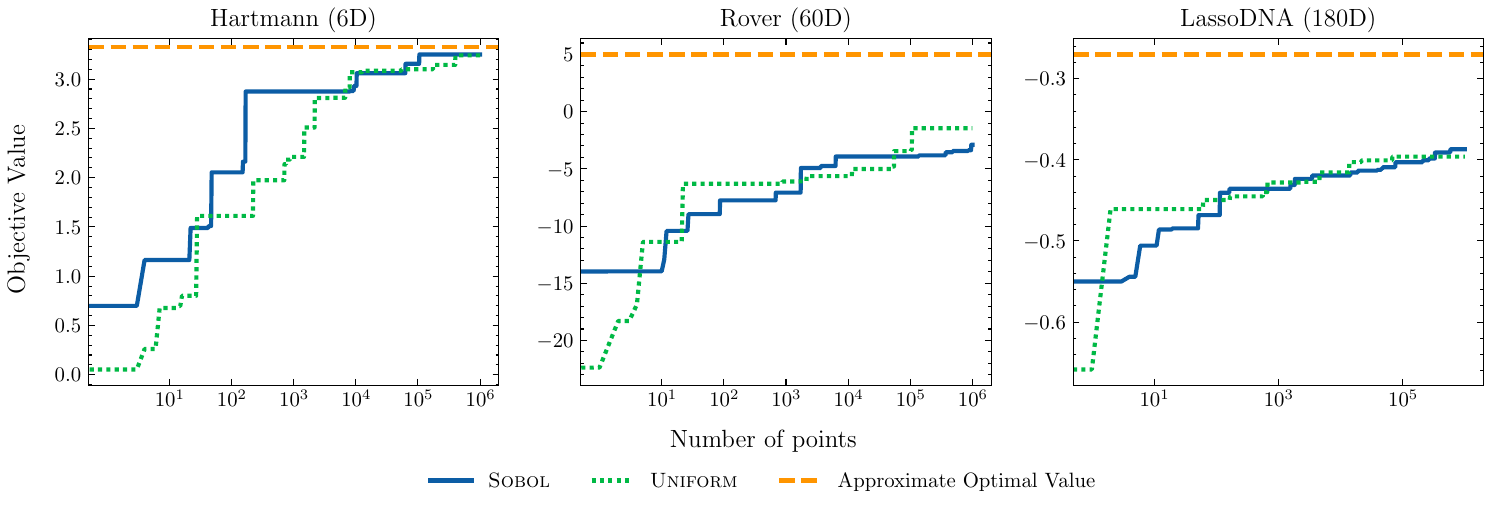}
  \caption{{\bf Global space filling sequence fail to fill the space of high-dimensional problems}. We track the best observed objective value as the number of space filling points increases, up to a budget of $10^6$ points. Aside from low-dimensional problems, the best found point remains far from the (approximate) optimal.}
  \label{fig:curse_dimensionality}
\end{figure}

\section{LOCALITY}\label{appendix:locality}
We attempt to characterize the locality of chosen methods in Figures \ref{fig:incumbent_dist} and \ref{fig:tsp_tour}.
For Fig. \ref{fig:tsp_tour}, we solve the Travelling Salesman Problem using a greedy approximate heuristic.
In both cases, a lower value may be correlated with a more local method.
For example, perhaps queries are close to the incumbent, indicating exploitation, or the points queried over time are not too far from each other.
We ultimately find that \methodnameshort\ does not exhibit any more local behavior than other popular methods and find that \textsc{Pathwise} appears to be the most exploratory.

\begin{figure}[t]
  \centering
  \includegraphics[width=\linewidth]{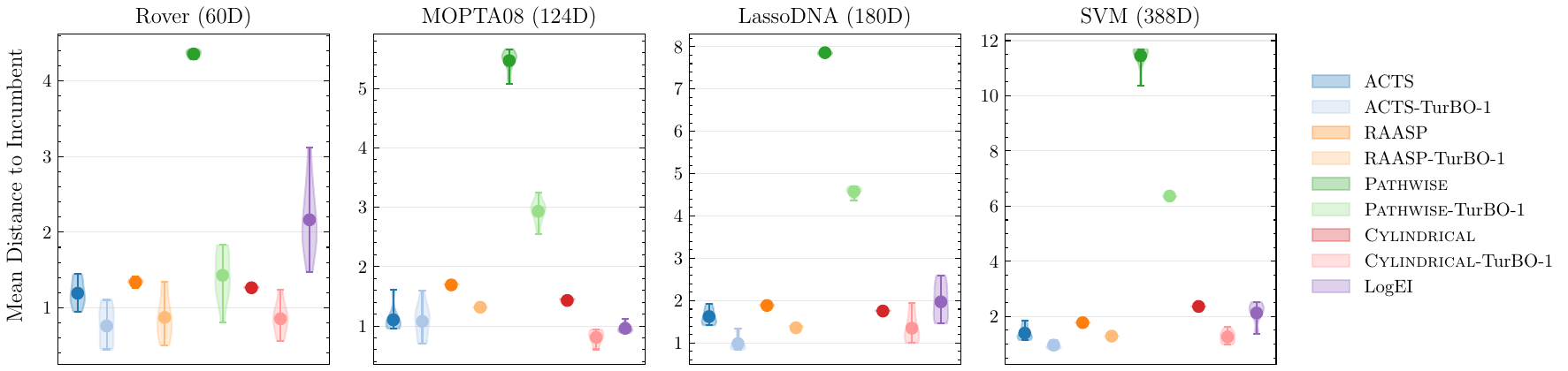}
  \caption{Mean euclidean distance ($\pm$ one standard error) from $\vx_t$ to $\vx_{t-1}$ averaged over $t$ and 10 runs (arbitrary units).}
  \label{fig:incumbent_dist}
\end{figure}

\begin{figure}[t]
  \centering
  \includegraphics[width=\linewidth]{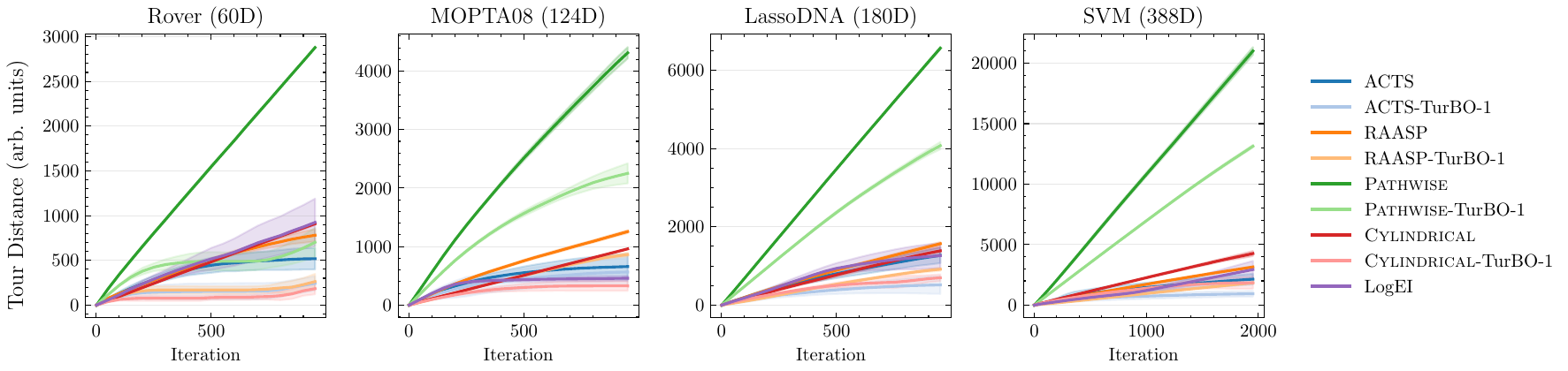}
  \caption{Travelling Salesman Problem tour ($\pm$ one standard error) for several selected objectives and acquisitions (arbitrary units), averaged over 10 runs.}
  \label{fig:tsp_tour}
\end{figure}

\section{REDUCTION IN GRADIENT UNCERTAINTY}\label{appendix:gradient_uncertainty}
As we acquire data over time, especially if we've been stuck at $\vx_0$ for some time, our estimation of the gradient will improve, even without explicit gradient observations.
We reinforce this in Figure \ref{fig:gradient_variance} and Table \ref{tab:gradient_variance_sign}.
Figure \ref{fig:gradient_variance} shows that the magnitude of the variance of the gradient decreases over time.
This is intuitive: as queries fail to provide a new incumbent, this information is used to improve the fit of the GP, thus leading to better gradient estimates.
However, the ACTS search space depends only on the direction, and thus the randomness of the sign of the gradient is also important.
In Table \ref{tab:gradient_variance_sign}, we compute the fraction of times each gradient sample dimension has a positive/negative sign, average the result over all dimensions, and compute the minimum fraction.
A smaller number indicates less disagreement in the sign.
In both cases, we find that ACTS is successful in leveraging the GP to reduce the uncertainty in the Thompson sample gradients.

\begin{figure}[t]
  \centering
  \includegraphics[width=\linewidth]{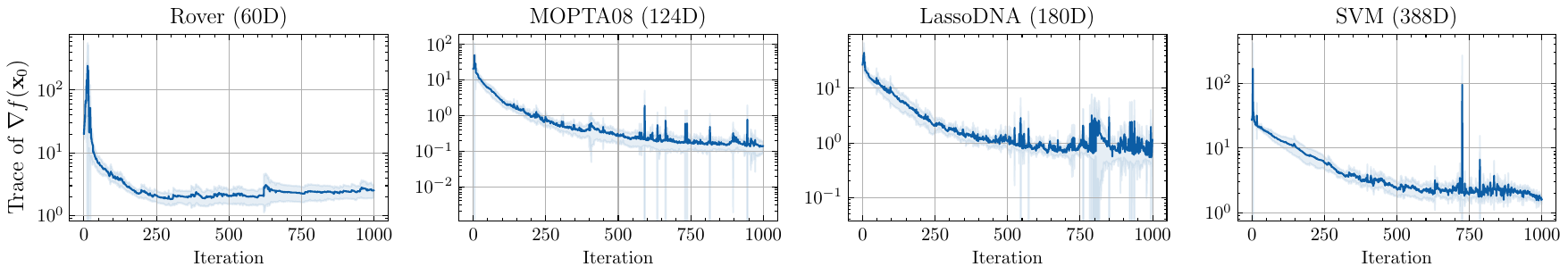}
  \caption{{\bf \methodnameshort\ reduces the uncertainty appreciably with new observations}. We aggregate over 10 runs, where the uncertainty in the gradient is measured by the trace of the covariance of the gradient distribution.}
  \label{fig:gradient_variance}
\end{figure}

\begin{table}[t]
  \centering
  \begin{tabular}[]{@{}lllllll@{}}
    \toprule\noalign{}
    & \(t=50\) & \(t=100\) & \(t=200\) & \(t=500\) & \(t=750\) &
    \(t=1000\) \\
    \midrule\noalign{}
    Rover (60D) & 0.323 & 0.345 & 0.281 & 0.195 & 0.198 & 0.207 \\
    MOPTA08 (124D) & 0.415 & 0.383 & 0.272 & 0.169 & 0.128 & 0.110 \\
    SVM (388D) & 0.470 & 0.439 & 0.384 & 0.334 & 0.271 & 0.246 \\
    LassoDNA (180D) & 0.446 & 0.407 & 0.376 & 0.213 & 0.233 & 0.240 \\
    \bottomrule\noalign{}
  \end{tabular}
  \caption{{\bf ACTS search space reduces uncertainty while also permitting exploration.} We compute the average minimum fraction of gradient samples that point in the minority direction. A number $p$ in this table can be interpreted to say that, on average, $p$ fraction of gradient samples point in the minority direction for any given dimension.}
  \label{tab:gradient_variance_sign}
\end{table}

\section{ABLATION STUDY OF \methodnameshort\ SEARCH SPACE}
\label{appx:ablation_search_space}
We recall our chosen subspace, a cone centered at the incumbent $\vx_0$ whose rays point in the positive scaled directions of the gradient $\nabla f(\vx_0)$:
\begin{equation*}
  \mathcal{T}_{\nabla f(\vx_0)} = \{\vx_0 + \vv\odot\nabla f(\vx_0) \: \mid \: \mathbf{0} \preceq \vv \in \mathbb{R}^d \} \cap \inputspace.
\end{equation*}
It is natural to consider the case when $\vv = v \boldsymbol{1}$ for some $v \geq 0$.
That is, we search in the 1D space induced solely by the direction of gradient.
Further, we also consider some sparsity in this 1D search: $[\vv]_j = v [\vb]_j$, where $[\vb]_j$ is sampled via Eq.~\ref{eqn:adaptive_raasp} and $v \geq 0$ for $j=1,\ldots,d$.
That is, we apply an adaptive RAASP-style mask to the gradient to induce search sparsity.
In Figure \ref{fig:ablation_subspace}, we consider these methods as ``Cone'', ``Line'' and ``Line w/ Random Subspace'', respectively and their use with TuRBO trust regions.
We allow the maximum value of $v$ to take the half-hypotenuse of the hypercube of the global space or trust region.
The 1D line subspaces are not as performant as when using the cone subspace.
In Figure \ref{fig:ls_sample_quality} we observe that when compared to Figure \ref{fig:sample_quality}, the 1D line search space is able to find strong TS maximizers, even compared to \methodnameshort\, however the objective values are highly concentrated.
Thus these 1D approaches may suffer from overexploitation.

\begin{figure}[t]
  \centering
  \includegraphics[width=\linewidth]{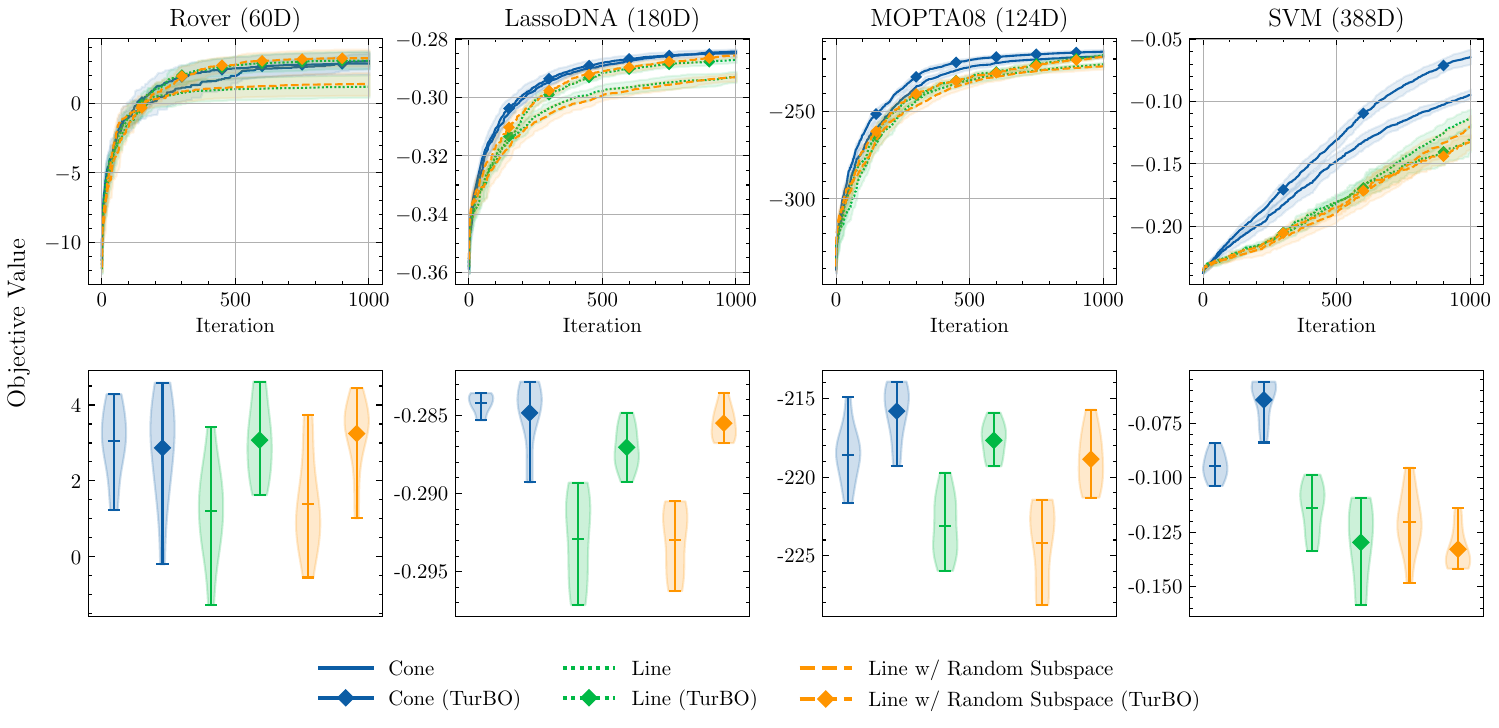}
  \caption{{\bf Optimization performance of \methodnameshort\ when ablating on different search spaces}. While the 1D line search spaces provide a comparable level of performance in some cases, it fails to be as performant as the proposed cone search space.}
  \label{fig:ablation_subspace}
\end{figure}

\begin{figure}[t]
  \centering
  \includegraphics[width=\linewidth]{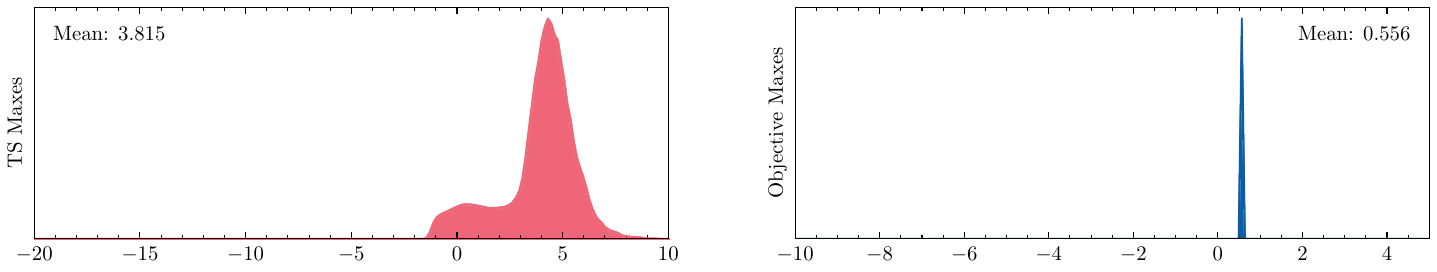}
  \caption{Candidate point quality of 1D line subspace. (See Figure \ref{fig:sample_quality} to compare). While the observed TS maxes are high, the objective function values at these maxima are highly concentrated.}
  \label{fig:ls_sample_quality}
\end{figure}

\section{ABLATION STUDY OF \methodnameshort\ WITH RAASP}
We ablate on different parameters of the adaptive RAASP-style policy, defined in Eq.~\ref{eqn:adaptive_policy}, where we recall one of the probabilities as $P = c\gamma$, where $c=20$ and $\gamma = |\nabla f(\vx_0)|_j^2/|\nabla f(\vx_0)|_2^2$.
We first consider the degree to which we favour large elements of the gradient norm, where we settled on the ``$L2$'' or squared normalized formulation.
We consider different ``$Lp$''-style formulations, where $\gamma = |\nabla f(\vx_0)|_j^p/|\nabla f(\vx_0)|_p^p$ for $p \in \{1, 2, 3\}$.
We further consider ``Top-K'', where the top $K=\min\{20,d\}$ dimensions are always perturbed, ranked by magnitude. Lastly, we consider using the softmax probabilities on the element-wise magnitude of the gradient, $\gamma = \mathrm{softmax}(|\nabla f(\vx_0)|)$ where $|\nabla f(\vx_0)| = [|\nabla f(\vx_0)|_j]_{j=1}^d$.

In Figs.~\ref{fig:ablation_probability_global} and \ref{fig:ablation_probability_turbo}, we observe mild sensitivity to the exact choice of $\gamma$, except when using the ``Top-K'' approach. ``$L3$'' may be more performant in some cases. Using TuRBO appears to further reduce the sensitivity.

\begin{figure}[t]
  \centering
  \includegraphics[width=\linewidth]{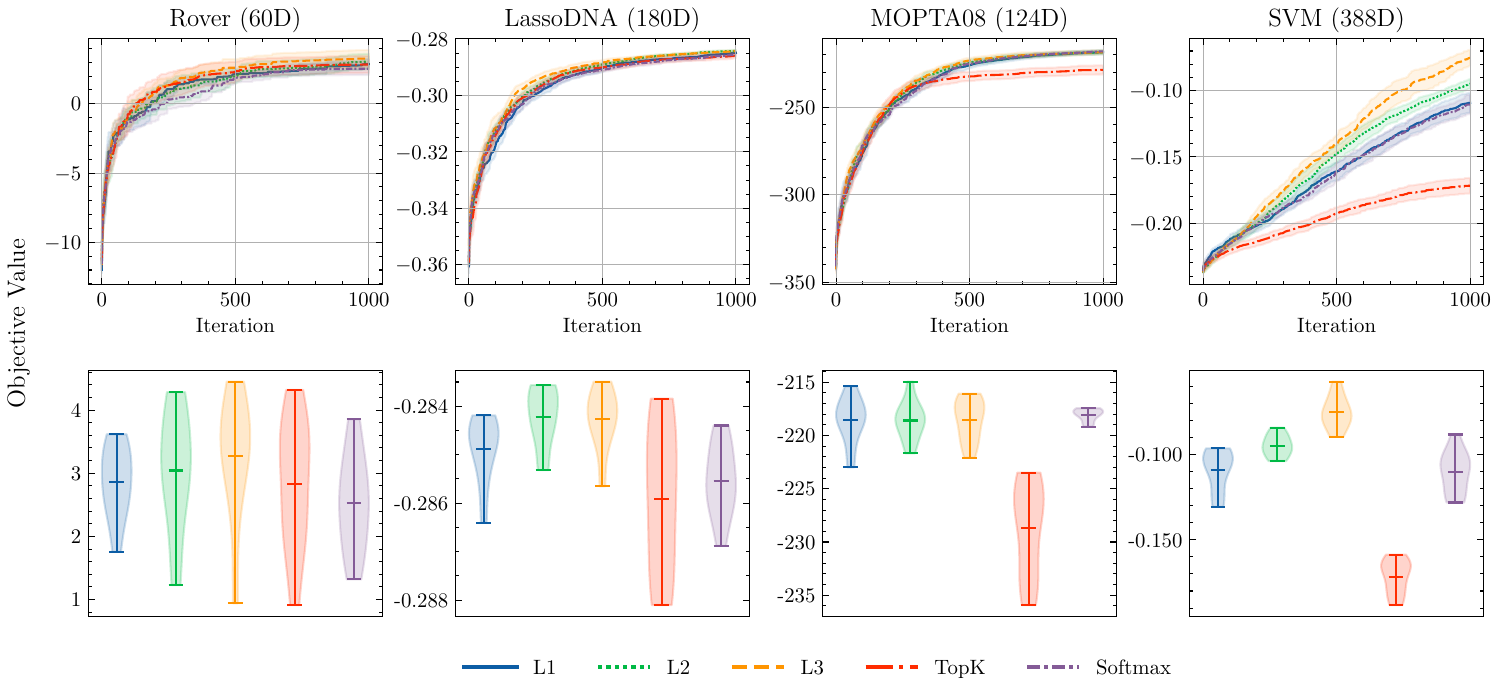}
  \caption{{\bf Increased preference to large-magnitude gradient dimensions improves \methodnameshort}. Except when using ``TopK'', the final performance remains competitive.}
  \label{fig:ablation_probability_global}
\end{figure}

\begin{figure}[t]
  \centering
  \includegraphics[width=\linewidth]{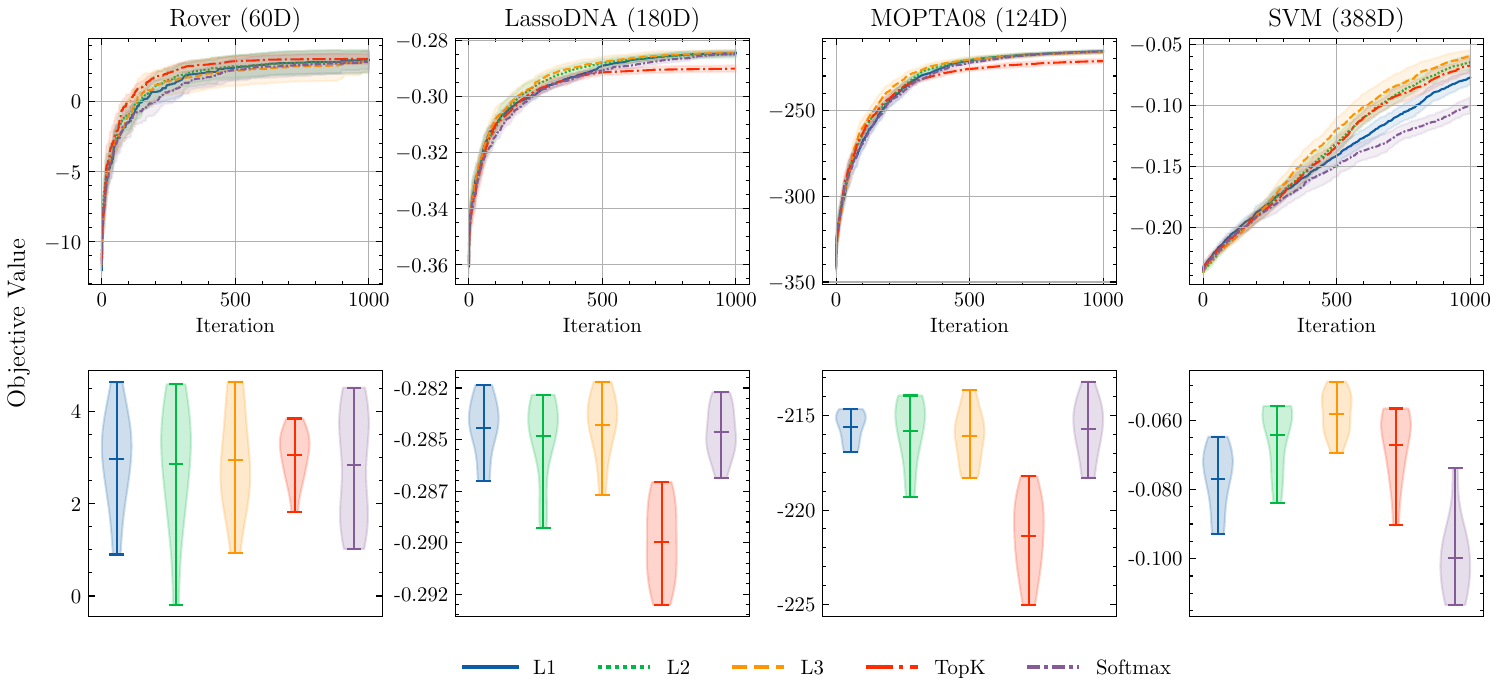}
  \caption{{\bf Increased preference to large-magnitude gradient dimensions improves \methodnameshort}. However, the effect is less pronounced when using TuRBO.}
  \label{fig:ablation_probability_turbo}
\end{figure}

Finally, we ablate on the choice $c$, which controls the average number of perturbed dimensions, to scale with the dimensionality in Figs.~\ref{fig:ablation_coefficient_global} and \ref{fig:ablation_coefficient_turbo}.
We observe that when perturbing in every dimension, the performance varies more relative to $c=20$, potentially indicating that too many dimensions are perturbed.
On the other hand, setting $c=d/10$ has no appreciable effect.
Similarly, we observe when using TuRBO trust regions that sensitivity to this parameter is less emphasized.

\begin{figure}[t]
  \centering
  \includegraphics[width=\linewidth]{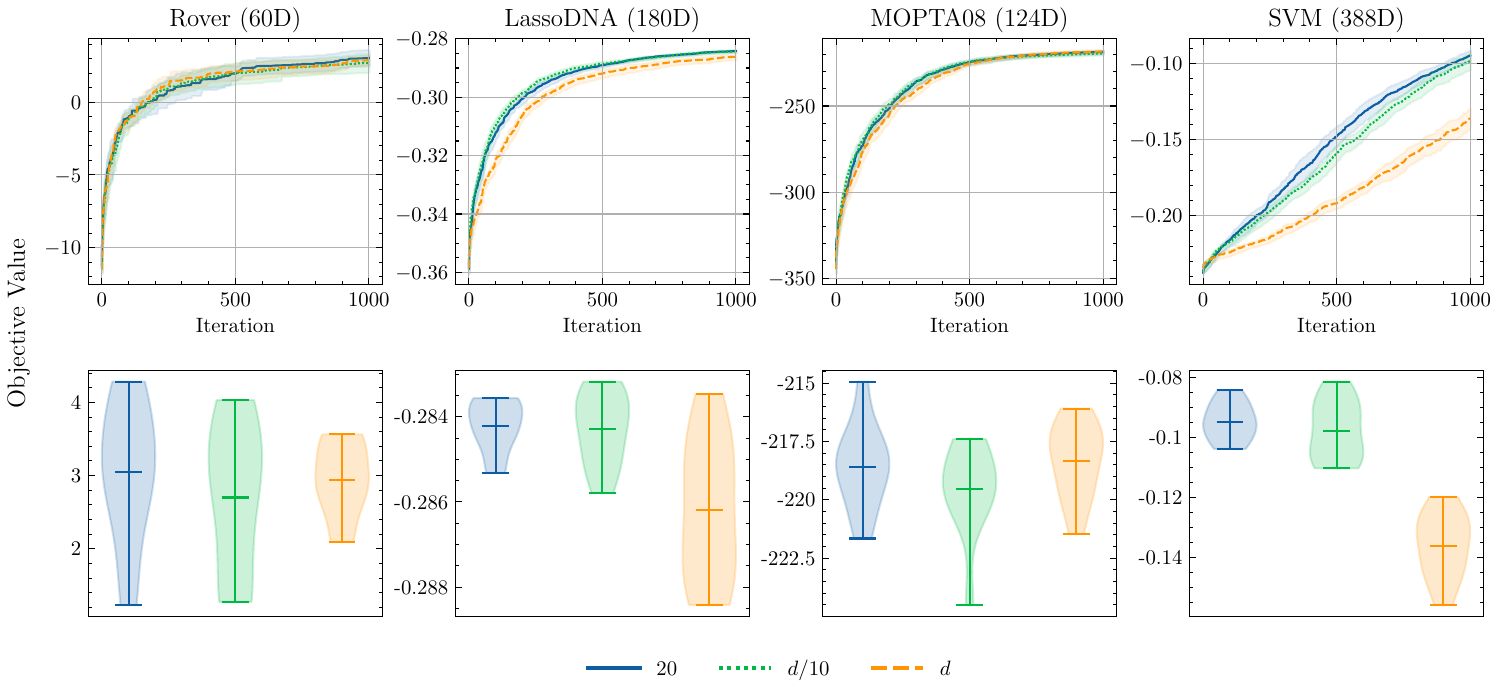}
  \caption{{\bf Optimization performance when ablating on the effective number of dimensions perturbed by RAASP.}}
  \label{fig:ablation_coefficient_global}
\end{figure}

\begin{figure}[t]
  \centering
  \includegraphics[width=\linewidth]{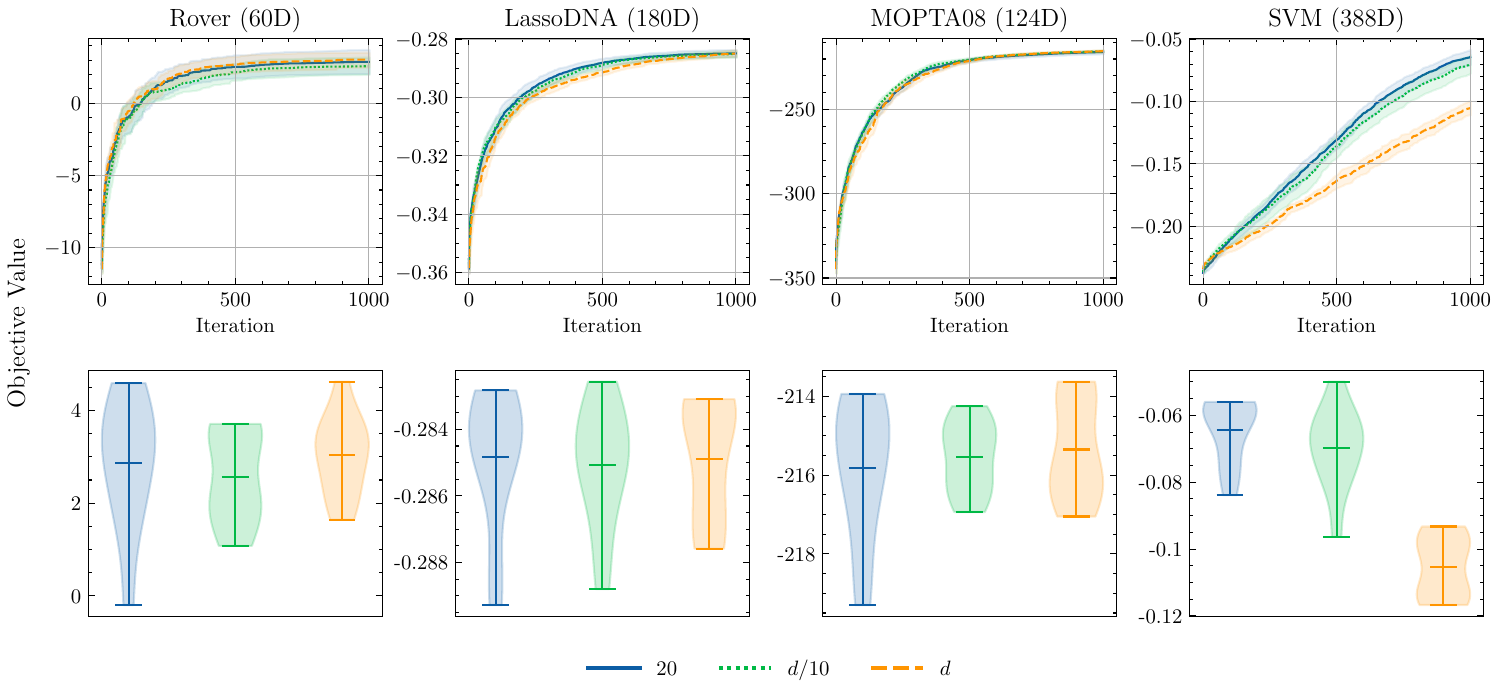}
  \caption{{\bf Optimization performance when ablating on the effective number of dimensions perturbed by RAASP when using TuRBO.}}
  \label{fig:ablation_coefficient_turbo}
\end{figure}

\end{document}